\documentclass[lettersize,journal]{IEEEtran}
\usepackage{amsmath,amssymb,amsfonts,amsthm}
\usepackage{array}
\usepackage[caption=false,font=footnotesize,labelfont=rm,textfont=rm]{subfig}
\usepackage{textcomp}
\usepackage{stfloats}
\usepackage{url}
\usepackage{cases}
\allowdisplaybreaks[2]

\usepackage{adjustbox}
\newtheorem{theorem}{Theorem}
\newtheorem{lemma}{Lemma}
\newtheorem{corol}{Corollary}
\newtheorem{defi}{Definition}
\usepackage{multirow}
\usepackage{xcolor}
\usepackage{verbatim}
\usepackage{graphicx}
\usepackage{setspace}
\usepackage{cite}
\usepackage{threeparttable}
\usepackage{booktabs}
\allowdisplaybreaks[1]
\usepackage{accents}

\usepackage{hyperref}
\usepackage{algorithm}
\usepackage{algpseudocodex}

\def\BibTeX{{\rm B\kern-.05em{\sc i\kern-.025em b}\kern-.08em
    T\kern-.1667em\lower.7ex\hbox{E}\kern-.125emX}}
\usepackage{balance}
\begin{document}
\bstctlcite{setting}

\title{A Theoretical Analysis of State Similarity Between Markov Decision Processes}

\author{Zhenyu~Tao,~\IEEEmembership{Graduate Student Member,~IEEE},
        Wei~Xu,~\IEEEmembership{Fellow,~IEEE},\\
        and Xiaohu~You,~\IEEEmembership{Fellow,~IEEE}

\thanks{This paper was presented in part at the 2025 NeurIPS Conference \cite{GBSM}.}
\thanks{Z. Tao, W. Xu, and X. You are with the National Mobile Communications Research Lab, Southeast University, Nanjing 210096, China, and also with the Pervasive Communication Research Center, Purple Mountain Laboratories, Nanjing 211111, China (email: \{zhenyu\_tao, wxu, xhyu\}@seu.edu.cn). Xiaohu You is the corresponding author of this paper.}
}


\maketitle

\begin{abstract}
  The bisimulation metric (BSM) is a powerful tool for analyzing state similarities \textit{within} a Markov decision process (MDP), revealing that states closer in BSM have more similar optimal value functions. While BSM has been successfully utilized in reinforcement learning (RL) for tasks like state representation learning and policy exploration, its application to state similarity \textit{between} multiple MDPs remains challenging. Prior work has attempted to extend BSM to pairs of MDPs, but a lack of well-established mathematical properties has limited further theoretical analysis between MDPs. In this work, we formally establish a generalized bisimulation metric (GBSM) for measuring state similarity between \textit{arbitrary pairs} of MDPs, which is rigorously proven with three fundamental metric properties, i.e., GBSM symmetry, inter-MDP triangle inequality, and a distance bound on identical spaces. Leveraging these properties, we theoretically analyze policy transfer, state aggregation, and sampling-based estimation across MDPs, obtaining explicit bounds that are strictly tighter than existing ones derived from the standard BSM. Additionally, GBSM provides a closed-form sample complexity for estimation, improving upon existing asymptotic results based on BSM. Numerical results validate our theoretical findings and demonstrate the effectiveness of GBSM in multi-MDP scenarios.
\end{abstract}

\begin{IEEEkeywords}
Markov decision process (MDP), reinforcement learning (RL), bisimulation metric (BSM).
\end{IEEEkeywords}

\section{Introduction}
Markov decision processes (MDPs) serve as a foundational framework for modeling decision-making problems~\cite{6587286}, particularly in reinforcement learning (RL)~\cite{sutton1998reinforcement,9760515}. To enable efficient analysis of MDPs, Ferns~\textit{et~al.}~\cite{10.5555/1036843.1036863} proposed the bisimulation metric (BSM) based on the Wasserstein distance, also known as the Kantorovich-Rubinstein metric, to quantify state similarity in a policy-independent manner. BSM provides theoretical guarantees that states closer under this metric exhibit more similar optimal value functions in RL. Meanwhile, BSM is a pseudometric~\cite{collatz2014functional} satisfying symmetry, triangle inequality, and indiscernibility of identicals. These three properties, combined with BSM’s measuring capability on optimal value functions, have driven its wide applications across diverse RL tasks. It has been successfully employed in state aggregation~\cite{10.5555/3020419.3020441,DBLP:journals/tmlr/KemertasJ22,DBLP:conf/aaai/0008ZY023a}, representation learning~\cite{10.5555/3666122.3667353,DBLP:conf/iclr/0001MCGL21,10.5555/3540261.3540625, 10560059}, policy exploration and generalization\cite{10.5555/3666122.3667805,DBLP:conf/iclr/AgarwalMCB21,DBLP:conf/nips/ChenP22}, goal-conditioned RL~\cite{pmlr-v162-hansen-estruch22a}, safe RL~\cite{10829536}, offline RL~\cite{DBLP:conf/icml/GuZCLHA22,9803237,DBLP:conf/icml/DadashiRVHPG21}, model predictive control~\cite{DBLP:conf/iclr/Shimizu23}, etc.

However, since BSM is inherently defined over a single MDP, its direct application to theoretical analyses involving multiple MDPs faces notable obstacles. For instance, Phillips~\cite{phillips2006knowledge} applied BSM to policy transfer by constructing a disjoint union of the source and target MDPs’ state spaces. While this operation allows inter-MDP comparisons through BSM, it also enforces disjoint transition probability distributions for states originating from different MDPs, hindering further simplifications and analysis. It necessitates iterative calculation of distances across the entire state space, leading to prohibitive computational costs in deep RL tasks as noted in~\cite{10.5555/1577069.1755839}. To compute state similarities in continuous or large-space discrete MDPs, Ferns~\textit{et~al.}~\cite{10.1137/10080484X} proposed a state similarity approximation method through state aggregation and sampling-based estimation. Although they proved the convergence of approximated state similarities to actual ones, their approach only derived a fairly loose approximation error bound and failed to obtain an explicit sample complexity~\cite{9570295}, i.e., the lower bound on the number of samples required to achieve the specified level of accuracy for an estimation error target. Specifically, the estimation error bound~\cite[equation 7.1]{10.1137/10080484X} depends on the former aggregation process, resulting in an asymptotic sample complexity rather than a closed-form expression. In addition, for representation learning, Zhang~\textit{et~al.}~\cite{DBLP:conf/iclr/0001MCGL21} and Kemertas and Aumentado-Armstrong~\cite{10.5555/3540261.3540625} leveraged BSM to establish value function approximation bounds under optimal and non-optimal policies, respectively, but only derived loose approximation bounds, particularly with large discount factors.

To address the limitation of standard BSM in multi-MDP analysis, several works have sought a modified BSM definition to evaluate state similarity between multiple MDPs~\cite{10.5555/1838206.1838401,10.5555/2936924.2936994,10.1007/s10994-022-06242-4}. Notably, when extended to multi-MDP scenarios, this modified version of BSM loses the original pseudometric properties of BSM, as compared states are in different MDPs. To the best of our knowledge, prior works have typically applied formulations from the single-MDP case directly to the multi-MDP domain, without rigorously reestablishing metric properties under multi-MDP settings. Specifically, Castro and Precup~\cite{10.5555/1838206.1838401} utilized its evaluation capability on optimal value functions to analyze policy transfer. Due to the lack of formal metric properties, the derived theoretical performance bound is limited to transferring only the optimal policy within the source MDP, and it can only reflect the effect of one-step action rather than the long-term impact of the transferred policy~\cite[Theorem 5]{10.5555/1838206.1838401}. 
While Song~\textit{et~al.}~\cite{10.5555/2936924.2936994} successfully employed such a modified BSM in assessing MDP similarities and improving the long-term reward in policy transfer, their investigation focused on empirical validation rather than in-depth theoretical analysis. Furthermore, a recent survey on state similarity measures between MDPs~\cite{10.1007/s10994-022-06242-4}, which highlighted the modified BSM as an effective approach, also emphasized the lack of theoretical guarantees in current multi-MDP similarity metrics. This raises the following two foundational questions:

\noindent\textit{Q1: Does the modified BSM possess any properties for computing state similarities between MDPs, akin to the pseudometric properties of standard BSM within a single MDP?}

\noindent\textit{Q2: If so, how can these properties facilitate the theoretical analysis involving multiple MDPs?
}

To answer these questions, we present a set of theoretical analyses on state similarities between MDPs. The main contributions are summarized as follows:
\begin{itemize}
    \item  We present a formal definition of the generalized BSM (GBSM) between MDPs (Definition~\ref{DefiBSM}), and rigorously establish its metric properties, including GBSM symmetry (Theorem~\ref{Theorem:3}), inter-MDP triangle inequality (Theorem~\ref{Theorem:4}), and the distance bound on identical spaces (Theorem~\ref{Theorem:5}). These properties serve as a direct generalization of the standard BSM's pseudometric properties, and as expected, when the compared MDPs are identical, the GBSM definition and its properties reduce to those of the original BSM.
    \item We apply GBSM in theoretical analyses of policy transfer, state aggregation, and sampling-based estimation of MDPs, yielding explicit bounds for policy transfer performance (Theorem~\ref{Theorem:6}), value function approximation (Theorem~\ref{Theorem:vfa}), and state similarity approximation (Theorems~\ref{Theorem:7} and \ref{Theorem:8}), respectively. Notably, when specialized to the case of identical MDPs, the GBSM framework provides strictly tighter error bounds for BSM, along with an explicit and closed-form sample complexity for approximation that advances beyond the existing asymptotic results. It also yields an explicit sample complexity for model-based RL as a direct consequence.
    \item To demonstrate the compatibility and practicality of GBSM, we showcase extensions of GBSM to other variants of BSM, such as lax BSM~\cite{10.5555/2981780.2981986} and on-policy BSM~\cite{Castro_2020}. Moreover, we propose an efficient algorithm by combining both aggregation and estimation methods to calculate GBSM in practical scenarios, where the state space is commonly large and exact probabilities are often inaccessible.
    \item  We validate the theoretical findings through experiments on both random Garnet MDPs and a practical simulation-to-reality (sim-to-real) RL task in wireless networks. Numerical results corroborate our theoretical findings.
\end{itemize}

\section{Preliminaries}
Before describing the details of our contributions, we give a brief review of the basics of RL, BSM, and the Wasserstein distance.

\textit{Reinforcement Learning:} We consider an MDP $\langle \mathcal{S}, \mathcal{A}, \mathbb{P}, R, \gamma\rangle$ defined by a finite state space $\mathcal{S}$, a finite action space $\mathcal{A}$, transition probability $\mathbb{P}(\tilde{s}|s,a)$ ($a\in\mathcal{A}$, $\{\tilde{s},s\}\in\mathcal{S}$, and $\tilde{s}$ denotes the next state), a reward function $R(s,a)\in[0,\bar{R}]$, and a discount factor $\gamma$. Policies $\pi(\cdot|s)$ are mappings from states to distributions over actions, inducing a value function recursively defined by $V^\pi(s)\!\!=\!\!\mathbb{E}_{a \sim \pi(\cdot|s)}\left[R(s,a)\!+\!\gamma \mathbb{E}_{\tilde{s} \sim \mathbb{P}(\cdot|s,a)}\left[V^\pi(\tilde{s})\right]\right]$. In RL, we are concerned with finding the optimal policy $\pi^*=\arg\max_\pi V^\pi$, which induces the optimal value function denoted by $V^*$.

\textit{Bisimulation Metric:} Different definitions of BSM exist in the literature~\cite{10.5555/1036843.1036863,10.5555/3020419.3020441,10.1137/10080484X}. In this paper, we adopt the formulation from~\cite{10.5555/3020419.3020441} and set the weighting constant to its maximum value $c=\gamma$. The BSM is then defined as:
\begin{align}
    d^{\sim}(s, s')=&\max_{a\in\mathcal{A}} \Big\{\big|R(s,a)-R(s',a)\big| \notag\\
    &\ +\gamma W_1(\mathbb{P}(\cdot|s,a), \mathbb{P}(\cdot|s',a) ;d^{\sim})\Big\},\forall \{s,s'\}\in \mathcal{S}.
\end{align}
Here, $W_1$ is the 1-Wasserstein distance, measuring the minimal transportation cost between distributions $\mathbb{P}(\cdot|s,a)$ to $\mathbb{P}(\cdot|s',a)$, with $d^{\sim}$ as the cost function. Ferns~\textit{et~al.}~\cite{10.5555/1036843.1036863} showed that this metric consistently bounds differences in the optimal value function, i.e., $|V^*(s)-V^*(s')|\leq d^{\sim}(s,s')$, and satisfies three pseudometric properties: (1) Symmetry: $d^{\sim}(s,s')=d^{\sim}(s',s)$, (2) Triangle inequality: $d^{\sim}(s,s')\leq d^{\sim}(s,s'')+d^{\sim}(s'',s')$, and (3) Indiscernibility of identicals: $s=s'\Rightarrow d^{\sim}(s,s')=0$. 

\textit{Wasserstein distance:} 
The Wasserstein distance~\cite{villani2021topics} is defined through the following primal linear program (LP):
\begin{equation}\label{LP1}
\begin{aligned}
 W_1(P, Q;d) =&\   \min _{\boldsymbol{\lambda}}\sum_{i=1}^{|\mathcal{S}_1|}\sum_{j=1}^{|\mathcal{S}_2|} \lambda_{i,j} d\left(s_i, s_j\right), \\
 \text { subject to } & \sum_{j=1}^{|\mathcal{S}_2|}\lambda_{i,j}=P\left(s_i\right),\ \forall \ i\  ;\\
    & \sum_{i=1}^{|\mathcal{S}_1|} \lambda_{i,j}=Q\left(s_j\right) ,\ \forall \ j\  ;\\
    & \lambda_{i,j} \geq 0,\ \forall \ i,j .
\end{aligned} 
\end{equation}
Here, $P$ and $Q$ are distributions on $\mathcal{S}_1$ and $\mathcal{S}_2$, respectively, and $s_i\in\mathcal{S}_1,s_j\in\mathcal{S}_2$. It represents the minimum transportation cost \cite{9517977} from $P$ to $Q$ under cost function $d:\mathcal{S}_1\times\mathcal{S}_2\rightarrow\mathbb{R}_+$, and is equivalent to the following dual LP according to the Kantorovich duality~\cite{kantorovich1958space}:
\begin{equation}\label{LP2}
    \begin{aligned}
    W_1(P, Q;d) =\  & \max_{\boldsymbol{\mu},\boldsymbol{\nu}}\sum_{i=1}^{|\mathcal{S}_1|} \mu_i P(s_i)-\sum_{j=1}^{|\mathcal{S}_2|} \nu_j Q(s_j),\\
    \text{subject to }& \ \mu_i-\nu_j\leq d(s_i,s_j),\ \forall \ i,j .
\end{aligned}
\end{equation}
\section{Generalized Bisimulation Metric}
We now present a formal definition of GBSM and derive its key metric properties.
\begin{defi}[\textbf{GBSM}\footnote{This extends the definition in our earlier conference version \cite{GBSM}, which was restricted to MDPs sharing the same action space. By adopting the Hausdorff metric, the formulation presented here allows for the evaluation of state similarity between arbitrary MDPs with potentially heterogeneous action spaces. Appendix~\ref{Appendix1.5} demonstrates that this new formulation and its associated bounds are strictly tighter than those of the original definition.}]\label{DefiBSM}
Given two MDPs $\mathcal{M}_1=\langle \mathcal{S}_1, \mathcal{A}_1, \mathbb{P}_1, R_1, \gamma\rangle$ and $\mathcal{M}_2=\langle \mathcal{S}_2, \mathcal{A}_2, \mathbb{P}_2, R_2, \gamma\rangle$, and a metric $d$ between states $s\in\mathcal{S}_1$ and $s'\in\mathcal{S}_2$, we define a function $\delta$ for state-action pairs as:
\begin{align}\label{eqdelta}
    \delta (d)((s,a), (s',a'))&=|R_1(s,a)-R_2(s',a')|\notag\\
    &\  +\gamma W_1(\mathbb{P}_1(\cdot|s,a), \mathbb{P}_2(\cdot|s',a') ;d),
\end{align}
where $a\in\mathcal{A}_1$ and $a'\in\mathcal{A}_2$. Then, the GBSM between states across two MDPs is defined as
\begin{align}\label{eq:defi-GBSM}
    d((s,\mathcal{M}_1), (s',&\mathcal{M}_2))=H(X_s,X_{s'};\delta (d)),
\end{align}
where $X_s=\{(s,a)|a\in\mathcal{A}_1\}$, $X_{s'}=\{(s',a')|a'\in\mathcal{A}_2\}$, and $H$ is the Hausdorff metric \cite{rockafellar1998variational} defined by 
\begin{equation}
 H(X,Y;d)=\max\{\max_{x\in X} \min_{y\in Y}d(x,y),\min_{x\in X}\max_{y \in Y}d(x,y)\}.
 \end{equation}
 \end{defi}
For notational simplicity, we use $d^{1\text{-}2}(s,s')$ to denote $d((s,\mathcal{M}_1), (s',\mathcal{M}_2))$, where the superscript $1\text{-}2$ indicates that the first state belongs to $\mathcal{M}_1$ and the second to $\mathcal{M}_2$. Then the existence of such a $d^{1\text{-}2}$ satisfying (\ref{eq:defi-GBSM}) is established by the following theorem.
\begin{figure*}[t]
\centering
\begin{tabular}{cc}
\begin{tabular}{c}
\subfloat[]{\includegraphics[height=0.19\textwidth]{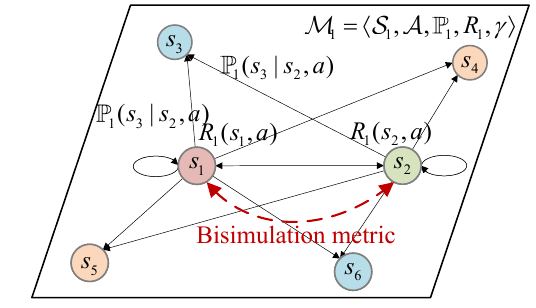}} \\
\subfloat[]{\includegraphics[height=0.19\textwidth]{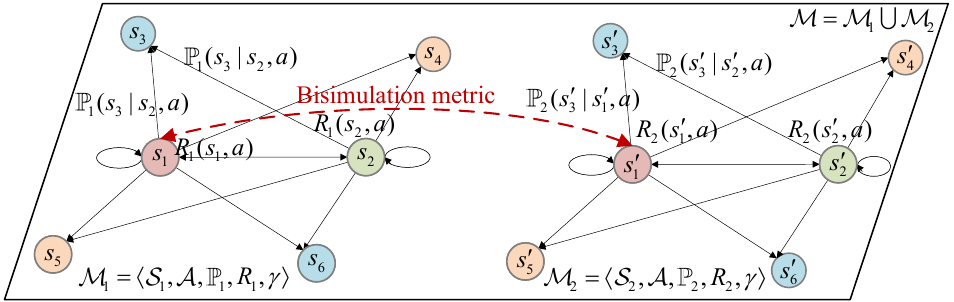}}
\end{tabular}
&
\adjustbox{valign=c}{\subfloat[]{\includegraphics[height=0.42\textwidth]{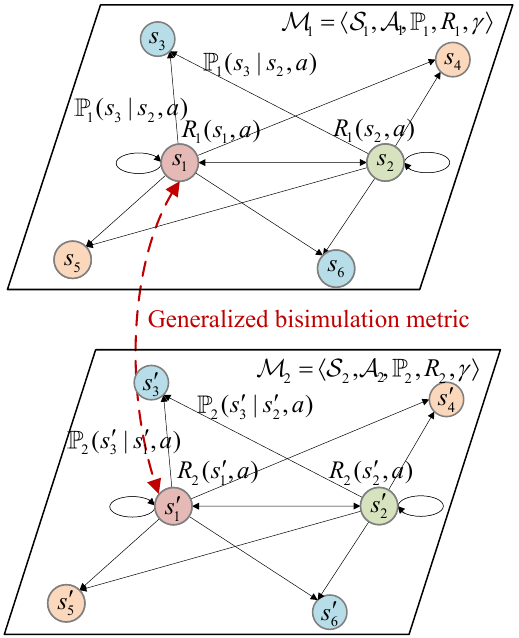}}} \\
\end{tabular}
\caption{Comparison between BSM and GBSM: (a) Analysis of state similarity within a single MDP using BSM. (b) Analysis of state similarity between two MDPs using BSM (requiring a disjoint union construction). (c) Direct analysis of state similarity between two MDPs using GBSM.}
\end{figure*}
\begin{theorem}[\textbf{Existence of GBSM}]\label{Theorem:1}
Let $d_0^{1\text{-}2}$ be a constant zero function and define
\begin{align}\label{eq:dn}
    d_{n}^{1\text{-}2}(s, s'&)=\ H(X_s,X_{s'};\delta (d_{n-1}^{1\text{-}2})),\ n\in\mathbb{N}.
 \end{align}
Then $d_n^{1\text{-}2}$ converges to the fixed point $d^{1\text{-}2}$ with $n\rightarrow\infty$.
\end{theorem}
\begin{proof}[Proof Sketch]
The existence of $d^{1\text{-}2}$ is established through the fixed-point theorem~\cite{tarski1955lattice} and the definition of the Wasserstein distance, similar to the proof of BSM in~\cite[Theorem 4.5]{10.5555/1036843.1036863}. See Appendix \ref{Appendix1} for complete proof. 
\end{proof}
Similar to BSM, which evaluates the state similarity through the optimal value function, GBSM naturally bounds differences in the optimal value function between two MDPs.
\begin{theorem}[\textbf{Optimal value difference bound between MDPs}]\label{Theorem:2}
Let $V_1^*$ and $V_2^*$ denote the optimal value functions in $\mathcal{M}_1$ and $\mathcal{M}_2$, respectively. Then the GBSM provides an upper bound for the difference between the optimal values for any state pair $(s,s')\in \mathcal{S}_1\times\mathcal{S}_2$:
    \begin{align}
        |V^*_1(s)-V^*_2(s')|\leq {d}^{1\text{-}2}(s,s').
    \end{align}
\end{theorem}
\begin{proof}[Proof Sketch]
We first construct a recursive form of the optimal value function, which starts from $V^{0}(s)=0$ and is defined by $V^{n}(s) =\ \max_{a}\Big\{R(s,a)+\gamma \mathbb{E}_{\tilde{s} \sim \mathbb{P}(\cdot|s,a)}\big[V^{(n-1)}(\tilde{s})\big]\Big\}$. The proof proceeds by induction on $n$. The key insight is that $\big( V_1^{n}(s_i) \big)_{i=1}^{|\mathcal{S}_1|}$ and $\big( V_2^{n}(s_j) \big)_{j=1}^{|\mathcal{S}_2|}$ form a feasible, but not necessarily the optimal, solution to the dual LP in (\ref{LP2}) for $W_1\big(\mathbb{P}_1(\cdot|s,a),\mathbb{P}_2(\cdot|s',a);d^{1\text{-}2}_n\big)$. See Appendix \ref{Appendix2} for complete proof.
\end{proof}

Now, we start to establish the three fundamental metric properties of GBSM, which we term GBSM symmetry, inter-MDP triangle inequality, and the distance bound on identical spaces. These properties align with the pseudometric properties of BSM, including symmetry, triangle inequality, and indiscernibility of identical.

\begin{theorem}[\textbf{GBSM symmetry}]\label{Theorem:3}
Let ${d}^{1\text{-}2}$ be the GBSM between $\mathcal{M}_1$ and $\mathcal{M}_2$, and ${d}^{2\text{-}1}$ be the GBSM with reversed MDP order, then we have
\begin{align}
   {d}^{1\text{-}2}(s,s')={d}^{2\text{-}1}(s',s),\  \forall\ (s,s')\in \mathcal{S}_1\times\mathcal{S}_2.
\end{align}
\end{theorem}
\begin{proof}
This property is readily proved through induction. We have $|R_1(s,a)-R_2(s',a')|=|R_2(s',a')-R_1(s,a)|$ for the base case. With the assumption of ${d}^{1\text{-}2}_n(s,s')={d}^{2\text{-}1}_n(s',s)$, we have $W_1\big(\mathbb{P}_1(\cdot|s,a),\mathbb{P}_2(\cdot|s',a')\allowbreak ;d^{1\text{-}2}_n\big)=W_1\big(\mathbb{P}_2(\cdot|s',a'),\mathbb{P}_1(\cdot|s,a);d^{2\text{-}1}_n\big)$, and from (\ref{eq:dn}) we have ${d}^{1\text{-}2}_{n+1}(s,s')={d}^{2\text{-}1}_{n+1}(s',s)$. It is therefore concluded that ${d}^{1\text{-}2}_n(s,s')={d}^{2\text{-}1}_n(s',s)$ for all $n\in\mathbb{N}$ and $(s,s')\in \mathcal{S}_1\times\mathcal{S}_2$. Taking $n\rightarrow \infty$ yields the desired result.
\end{proof}

\begin{theorem}[\textbf{Inter-MDP triangle inequality of GBSM}]\label{Theorem:4}
Given MDPs $\mathcal{M}_1=\langle \mathcal{S}_1, \mathcal{A}_1, \mathbb{P}_1, R_1, \gamma\rangle$, $\mathcal{M}_2=\langle \mathcal{S}_2, \mathcal{A}_2, \mathbb{P}_2, R_2, \gamma\rangle$, and $\mathcal{M}_3=\langle \mathcal{S}_3, \mathcal{A}_3, \mathbb{P}_3, R_3, \gamma\rangle$, GBSMs between the three MDPs satisfy
\begin{equation} \label{triangleieq}
    {d}^{1\text{-}2}(s,s')\leq{d}^{1\text{-}3}(s,s'')+{d}^{3\text{-}2}(s'',s'),\ 
\end{equation}
for any $(s,s',s'')\in \mathcal{S}_1\times\mathcal{S}_2\times\mathcal{S}_3$. Here, the GBSM between any two MDPs can be arbitrarily reversed according to its symmetry.
\end{theorem}
\begin{proof}[Proof Sketch]
We first prove the transitive property of inequality on the Wasserstein distance through Gluing Lemma~\cite{villani2009optimal}, and establish the inter-MDP triangle inequality through induction. See Appendix \ref{Appendix:Theorem:4} for complete proof.
\end{proof}

Since the identicals only exist within the same space, we establish the distance bound only when $\mathcal{M}_1$ and $\mathcal{M}_2$ share the same state and action spaces $\mathcal{S},\mathcal{A}$. This property is formulated as follows.

\begin{theorem}[\textbf{Distance bound on identical spaces}]\label{Theorem:5}
When $\mathcal{M}_1$ and $\mathcal{M}_2$ share the same state and action spaces $\mathcal{S},\mathcal{A}$, define
\begin{align}
    \delta_{\textnormal{TV}}((s,a),(s',a'&)) = \big|R_1(s,a)-R_2(s',a')\big|\notag\\
            & +\frac{\gamma \bar{R}}{1-\gamma} \textnormal{TV}\left(\mathbb{P}_1(\cdot|s,a), \mathbb{P}_2(\cdot|s',a') \right),
\end{align}
where TV represents the total variation distance defined by $\textnormal{TV}(P,Q)= \frac{1}{2} \sum_{s\in\mathcal{S}} \big|P(s)-Q(s)\big|$, then the GBSM satisfy
\begin{align}\label{TVineq}
   \max_{s}d^{\textnormal{1-2}}(s,s) \leq & \frac{1}{1-\gamma} \max_{s} H(X_s,X_s;\delta_{\textnormal{TV}}).
\end{align}
\end{theorem}
\begin{proof}[Proof Sketch]
This theorem is established by constructing a specific transportation plan between $\mathbb{P}_1(\cdot|s,a)$ and $\mathbb{P}_2(\cdot|s',a')$ and leveraging the boundedness of the reward function. See Appendix \ref{Appendix:Theorem:5} for complete proof.
\end{proof}
A direct consequence of Theorem \ref{Theorem:5} is that the right-hand side of inequality (\ref{TVineq}) becomes zero if $\mathcal{M}_1\!=\!\mathcal{M}_2$, since
\begin{align}
    H(X_s,X_s;&\delta_{\textnormal{TV}}) \leq  \max_a \big\{ \big|R_1(s,a)-R_2(s,a)\big|\notag\\
            & +\frac{\gamma \bar{R}}{1-\gamma} \textnormal{TV}\left(\mathbb{P}_1(\cdot|s,a), \mathbb{P}_2(\cdot|s,a) \right)\big\} =0.
\end{align}
It indicates $d^{\textnormal{1-1}}(s,s)=d((s,\mathcal{M}_1), (s,\mathcal{M}_1))=0$, confirming the indiscernibility of identicals of GBSM when the compared objects (state-MDP pairs) are genuinely identical. We denote the term $H(X_s,X_s;\delta_{\textnormal{TV}})$ in Theorem \ref{Theorem:5} as $d^{\textnormal{1-2}}_\textnormal{TV}(s,s)$ in the following.

Thus far, we present a formal definition of GBSM with rigorously proved metric properties. Notably, when all compared MDPs are identical, GBSM reduces to lax BSM~\cite{10.5555/2981780.2981986} that computes state similarity within one MDP. In this case, the three fundamental properties of GBSM proposed in this section reduces to the original pseudometric properties of lax BSM.
\section{Applications of GBSM in Multi-MDP Analysis}
To demonstrate the effectiveness of GBSM in multi-MDP scenarios, we apply it to theoretical analyses of policy transfer, state aggregation, and sampling-based estimation of MDPs.
\subsection{Performance Bound of Policy Transfer Using GBSM}
Using GBSM, we analyze policy transfer from a source MDP $\mathcal{M}_1$ to a target MDP $\mathcal{M}_2$ and derive a theoretical performance bound for the transferred policy. This bound takes the form of a regret which is defined as the expected discounted reward loss incurred by following the transferred policy instead of the optimal one~\cite{kaelbling1996reinforcement}. Specifically, it is a weighted sum of the GBSM between the two MDPs and the regret within the source MDP itself, formulated by the following.
\begin{theorem}[\textbf{Regret bound on policy transfer}]\label{Theorem:6}
 Let $\pi$ be a policy on $\mathcal{M}_1$. Suppose we transfer $\pi$ to $\mathcal{M}_2$ using a state mapping $f: \mathcal{S}_2 \to \mathcal{S}_1$ and an action mapping $g: \mathcal{A}_1 \to \mathcal{A}_2$, such that the agent in $\mathcal{M}_2$ executes action $g(\pi(f(s')))$ at state $s'\in\mathcal{S}_2$. The regret of the transferred policy in $\mathcal{M}_2$ is bounded by:
 \begin{align}\label{Th6eq1}
    &\max_{s'\in\mathcal{S}_2}|V_2^*(s')-V_2^{\pi}(s')|\notag\\
    \leq & \big(\max_{s'\in\mathcal{S}_2}d^{\textnormal{1-2}}(f(s'),s')+\max_{\substack{s' \in \mathcal{S}_2 \\ a \in \mathcal{A}_1}}\delta(d^{\textnormal{1-2}})((f(s'),a),(s',g(a)))\notag\\
&+(1+\gamma)\max_{s\in\mathcal{S}_1}|V_1^*(s)-V_1^{\pi}(s)|\big)\big/\big(1-\gamma\big).
\end{align}
\end{theorem}
\begin{proof}[Proof Sketch]
The proof of (\ref{Th6eq1}) is similar to the proof in~\cite[Theorem 4.1]{phillips2006knowledge} and is conducted by replacing the BSM by GBSM. See Appendix \ref{Appendix3} for complete proof.
\end{proof}
This theorem reveals that the policy transfer regret between \textit{any} two MDPs is upper-bounded by three key factors: (1) the inherent similarity distance between source and target MDPs $d^{\textnormal{1-2}}(f(s'),s')$, (2) the alignment of MDPs under state and action mappings $\delta(d^{\textnormal{1-2}})((f(s'),a),(s',g(a)))$, and (3) the suboptimality of the policy in the source MDP $|V_1^*(s)-V_1^{\pi}(s)|$. Special cases of Theorem \ref{Theorem:6} yield the following refined bounds:
\begin{corol}[\textbf{Optimal action mapping for policy transfer}]\label{corolactionmapping}
Suppose the action mapping $g$ is chosen greedily with respect to $d^{1\textnormal{-}2}$. That is, for any state $s' \in \mathcal{S}_2$ and action $a \in \mathcal{A}_1$, let $g(a) = \operatorname*{arg\,min}_{a' \in \mathcal{A}_2} \delta((f(s'), a), (s', a'); d^{1\textnormal{-}2})$.\footnote{In cases where $\mathcal{M}_1$ and $\mathcal{M}_2$ share the same action space, this requirement is typically satisfied by the identity mapping $g(a)=a$.} Then the regret bound is tightened to:
\begin{align}\label{boundwithoptimalactionmapping}
    \max_{s'\in\mathcal{S}_2}|V_2^*(s')-V_2^{\pi}(s')| \leq & \frac2{1-\gamma}\max_{s'\in\mathcal{S}_2}d^{\textnormal{1-2}}(f(s'),s')\notag\\
&+\frac{1+\gamma}{1-\gamma}\max_{s\in\mathcal{S}_1}|V_1^*(s)-V_1^{\pi}(s)|.
\end{align}
\end{corol}
\begin{proof}
This result follows directly from Definition \ref{DefiBSM}.
\end{proof}
\begin{corol}[\textbf{Optimal state mapping for policy transfer}]
Assume the conditions of Corollary \ref{corolactionmapping} hold. If the state mapping is further chosen as $f(s')=\arg\min_{s\in\mathcal{S}_1}d^{\textnormal{1-2}}(s,s')$ for all $ s'\in \mathcal{S}_2$, the bound is further tightened to:
\begin{align}
    \max_{s'\in\mathcal{S}_2}|V_2^*(s')-V_2^{\pi}(s')| \leq& \frac2{1-\gamma}\max_{s'\in\mathcal{S}_2}\Big\{\min_{s\in\mathcal{S}_1}d^{\textnormal{1-2}}(s,s')\Big\}\notag\\
&+\frac{1+\gamma}{1-\gamma}\max_{s\in\mathcal{S}_1}|V_1^*(s)-V_1^{\pi}(s)|.
\end{align}
\end{corol}
\begin{proof}
This result is established by applying the assumption on $f(s')$ to (\ref{boundwithoptimalactionmapping}).
\end{proof}
In contrast to the approach of~\cite{phillips2006knowledge}, which constructs a disjoint union state space for analysis, we provide a similar theoretical bound by directly analyzing the relationship between the source and target MDPs. Crucially, this formulation naturally accommodates MDPs with heterogeneous state and action spaces. Furthermore, this direct method avoids a constant total variation distance, thereby enabling simplifications such as the bound based on $d^{\textnormal{1-2}}_\textnormal{TV}$, as well as the approximation method in the following section. In terms of computational efficiency, calculating BSM on the disjoint union of two MDPs renders a significant computational complexity scaling with $|\mathcal{S}_1\!+\mathcal{S}_2|^2$. In contrast, our GBSM directly computes state similarity between $\mathcal{M}_1$ and $\mathcal{M}_2$, with a reduced complexity scaling with $|\mathcal{S}_1|\!\,\cdot\!\,|\mathcal{S}_2|$.

\subsection{Approximation Methods and Corresponding Error Bounds}
When the state space is extensive and actual transition probabilities are inaccessible, approximation methods are necessary for the efficient computation of state similarities. In the single MDP scenario, Ferns~\textit{et~al.}~\cite{10.1137/10080484X} proposed a state similarity approximation (SSA) method based on state aggregation and sampling-based estimation. Let $\mathcal{U}\subseteq \mathcal{S}$ be a set of selected representative states, $[\,\cdot\,]:\mathcal{S}\rightarrow\mathcal{U}$ an aggregation mapping, $\tilde{\sigma}=\max_{s\in\mathcal{S}}\{d^{\sim}(s,[s])\}$ the maximum aggregation distance, and $K$ the number of samples used to empirically estimate each transition probability. The SSA error satisfies
\begin{align}\label{Eq:aggBSM}
    \max_{s,s'}|&d^{\sim}(s,s')-d^{\sim}_{\tilde{\sigma},K}([s],[s'])|\leq \frac{2\tilde{\sigma} (2+\gamma)}{1-\gamma} \notag\\
&+ \frac{2 \gamma}{1-\gamma}\max_{a,s} W_1\Big([\hat{\mathbb{P}}](\cdot|[s],a), [\mathbb{P}](\cdot|[s],a) ;d^{\sim}\Big).
\end{align}
Here, $d^{\sim}_{\tilde{\sigma},K}$ denotes the BSM on the approximated MDP, $[\mathbb{P}]$ denotes the transition probability between aggregated states, and $[\hat{\mathbb{P}}]$ represents its empirical counterparts estimated from $K$ samples. However, the BSM-based aggregation error bound ${2\tilde{\sigma} (2+\gamma)}/{(1-\gamma)}$ is fairly loose, and the sample complexity for the estimation error is limited to asymptotic expressions. 

Apart from approximating state similarities, it is also crucial to directly quantify the difference between optimal value functions within the original MDPs and their approximated counterparts in aggregated MDPs. Using a BSM-based analysis, Zhang~\textit{et~al.}~\cite{DBLP:conf/iclr/0001MCGL21} established a value function approximation (VFA) bound on this difference, given by $2\tilde{\sigma}/(1-\gamma)$. A key limitation of this bound, however, is its looseness in settings with a large discount factor $\gamma$.

To address this, we apply the GBSM to directly compute state similarities between the original MDPs and their aggregated/estimated counterparts. Beyond extending the approach in~\cite{10.1137/10080484X} to the multi-MDP setting, our GBSM-based analysis yields significantly tighter approximation bounds for both SSA and VFA, and provides an explicit and closed-form expression for the sample complexity.


\subsubsection{State Aggregation}
Given the previously defined $\mathcal{S}$, $\mathcal{U}$, and $[\,\cdot\,]$, the aggregated state space $[\mathcal{S}]$ is defined such that the reward function and transition probability of each state are replaced by those of its representative state, given by $R(s,a)=R([s],a)$ and $\mathbb{P}(\cdot|s,a)=\mathbb{P}(\cdot|[s],a)$ for all $s\in\mathcal{S}$. The aggregated transition probability is defined as $[\mathbb{P}](s'|s,a)=\sum_{s''\in\mathcal{S},[s'']=s'} \mathbb{P}(s''|s,a)$. Note that $[\mathbb{P}](s'|s,a)=0$ when $s'\notin\mathcal{U}$. With this construction, we define the aggregated MDP for $\mathcal{M}_1$ as $\mathcal{M}_{[1]}= \langle [\mathcal{S}_1], \mathcal{A}_1, [\mathbb{P}_1], R_1, \gamma\rangle$. First, we obtain the VFA bound directly from GBSM.
\begin{theorem}[\textbf{VFA error bound}]\label{Theorem:vfa}
Given an MDP $\mathcal{M}_1$ and its aggregated counterpart $\mathcal{M}_{[1]}$, the VFA bound is given by
\begin{equation} \label{vfa}
\max_{s\in\mathcal{S}_1}|V^*_1(s)-V^*_{[1]}(s)|\leq \sigma_1 \leq \tilde{\sigma}_1/(1-\gamma),
\end{equation}
where $\sigma_1 = \max\limits_{s\in\mathcal{S}_1}d^{1\textnormal{-}[1]}(s,[s])$ and $\tilde{\sigma}_1 = \max\limits_{s\in\mathcal{S}_1}d^{\sim}(s,[s])$.
\end{theorem}
\begin{proof}
The first inequality in (\ref{vfa}) follows directly from Theorem~\ref{Theorem:2}. The second is established by constructing an intermediate MDP between $\mathcal{M}_1$ and $\mathcal{M}_{[1]}$ and repeatedly applying the inter-MDP triangle inequality stated in Theorem~\ref{Theorem:4}. See Appendix \ref{Appendix:Theorem:vfa} for complete proof.
\end{proof}
This theorem demonstrates the significant tightness of the GBSM-based VFA bound $\sigma_1$ compared to the existing BSM-based bound $2\tilde{\sigma}_1/(1-\gamma)$ in~\cite{DBLP:conf/iclr/0001MCGL21}.
We are now ready to establish the aggregation error bound for SSA as follows.
\begin{theorem}[\textbf{SSA aggregation error bound}]\label{Theorem:7}
Given MDPs $\mathcal{M}_1$, $\mathcal{M}_2$ and their aggregated counterparts $\mathcal{M}_{[1]}$,$\mathcal{M}_{[2]}$, the SSA error bound is given by
\begin{align} \label{eq17}
\max_{s\in\mathcal{S}_1,s'\in\mathcal{S}_2}|d^{\textnormal{1-2}}(s,s')-d^{[1]\textnormal{-}[2]}(s,s')|&\leq \sigma_1 + \sigma_2 \notag \\ 
&\leq (\tilde{\sigma}_1+\tilde{\sigma}_2)/(1-\gamma).
\end{align}
\end{theorem}
\begin{proof}
This theorem is easily derived by combining Theorem~\ref{Theorem:4} and Theorem~\ref{Theorem:vfa}.
\end{proof}
When the compared MDPs are identical, i.e., $\mathcal{M}_2=\mathcal{M}_1=\langle \mathcal{S}, \mathcal{A}, \mathbb{P}, R, \gamma\rangle$, Theorem \ref{Theorem:7} reduces to the aggregation error bound in the single-MDP scenario as 
\begin{equation}
    \max_{s,s'}|d^{\textnormal{1-1}}(s,s')-d^{[1]\textnormal{-}[1]}(s,s')|\leq 2 \sigma_1 \leq 2\tilde{\sigma}_1/(1-\gamma),
\end{equation}
indicating significant tightness of the GBSM-based SSA bound $2 \sigma_1$ compared to the BSM-based one $2\tilde{\sigma}_1 (2+\gamma)/(1-\gamma)$~\cite{10.1137/10080484X}.

\subsubsection{Sampling-based Estimation}

To estimate a probability distribution $P$ through statistical sampling, we define the empirical distribution based on $K$ samples as $\hat{P}(x)=\frac{1}{K} \sum_{i=1}^K \delta_{X_i}(x)$, where $\{X_1, X_2, \ldots, X_K\}$, are $K$ independent points sampled from $P$ and $\delta$ denotes the Dirac measure at $X_i$ such that $\delta_{X_i}(x)=1$ if $x = X_i$ and $0$ otherwise. 
Then the empirical MDP for $\mathcal{M}_1$ is constructed by sampling $K$ points for each $\mathbb{P}_1(\cdot|s,a)$, defined by $\mathcal{M}_{\hat{1}}=\langle \mathcal{S}_1, \mathcal{A}_1, \hat{\mathbb{P}}_1, R_1, \gamma\rangle$. The estimation error bound is derived as follows.

\begin{theorem}[\textbf{SSA estimation error bound}]\label{Theorem:8}
Given MDPs $\mathcal{M}_1$, $\mathcal{M}_2$ and their empirical counterparts $\mathcal{M}_{\hat{1}}$,$\mathcal{M}_{\hat{2}}$, the SSA error bound is given by
\begin{align}\label{the8eq1}
   \max_{s,\!s'}|d^{\textnormal{1-2}}(s,\!s')\!-\!d^{\hat{1}\textnormal{-}\hat{2}}(s,s')|\!\leq \! \max_{s}d^{1\textnormal{-}\hat{1}}(s,\!s) \!+\!\max_{s'}d^{2\textnormal{-}\hat{2}}(s'\!,\!s').
\end{align}
To reach an error less than $\epsilon$ with a probability of $1-\alpha$, the sample complexity is given by
\begin{align}\label{samplecomplexity}
   K \geq -\ln(\alpha/2) \frac{\gamma^2\bar{R}^2|\mathcal{S}_{\cdot}|^2}{2\epsilon^2(1-\gamma)^4},
\end{align}
for each state-action pair in $\mathcal{M}_1$ (where $|\mathcal{S}_{\cdot}|=|\mathcal{S}_1|$) and $\mathcal{M}_2$ (where $|\mathcal{S}_{\cdot}|=|\mathcal{S}_2|$).
\end{theorem}
\begin{proof}
    Inequality (\ref{the8eq1}) is easily obtained from Theorem \ref{Theorem:4}. In terms of the sample complexity, we apply Theorem \ref{Theorem:5} and obtain
    \begin{align}
        &\max_{s,s'}|d^{\textnormal{1-2}}(s,s')-d^{\hat{1}\textnormal{-}\hat{2}}(s,s')| \notag \\
        \leq\ &\max_{s} d^{\textnormal{1-}\hat{1}}(s,s) + \max_{s'} d^{\textnormal{2-}\hat{2}}(s',s') \notag\\
        \leq\ &\big(\max_{s} d^{\textnormal{1-}\hat{1}}_\textnormal{TV}(s,s) + \max_{s'} d^{\textnormal{2-}\hat{2}}_\textnormal{TV}(s',s')\big)\big/\big( 1-\gamma \big)\notag\\
        \leq\ &\frac{\gamma \bar{R}}{(1-\gamma)^2} \bigg(\max_{s,a}\textnormal{TV}\big(\mathbb{P}_1(\cdot|s,a), \hat{\mathbb{P}}_1(\cdot|s,a) \big) \notag\\
        &\qquad\qquad+ \max_{s',a}\textnormal{TV}\big(\mathbb{P}_2(\cdot|s',a), \hat{\mathbb{P}}_2(\cdot|s',a) \big)\bigg)\notag \\
        =\ & \frac{\gamma \bar{R}}{2(1\!-\!\gamma)^2} \Bigg(\! \max_{s,a} \Bigg\{\! \sum_{\tilde{s}\in \mathcal{S}_1} \!\big|\mathbb{P}_1(\tilde{s}|s,a)\!-\! \hat{\mathbb{P}}_1(\tilde{s}|s,a)\big|\! \Bigg\} \!\notag\\
    &\qquad\qquad\ +\!  \max_{s'\!,a} \Bigg\{\! \sum_{\tilde{s}\in \mathcal{S}_2} \!\big|\mathbb{P}_2(\tilde{s}|s'\!,a)\!-\! \hat{\mathbb{P}}_2(\tilde{s}|s'\!,a)\big| \!\Bigg\}\!\Bigg).
    \end{align}
    To ensure the estimation error remains below $\epsilon$, we require $|\mathbb{P}_1(\tilde{s}|s,a)- \hat{\mathbb{P}}_1(\tilde{s}|s,a)|\leq \frac{\epsilon(1-\gamma)^2}{\gamma \bar{R}|\mathcal{S}_1|}$ and $|\mathbb{P}_2(\tilde{s}|s,a)- \hat{\mathbb{P}}_2(\tilde{s}|s,a)|\leq \frac{\epsilon(1-\gamma)^2}{\gamma \bar{R}|\mathcal{S}_2|}$. Next, by applying the Hoeffding's inequality~\cite{hoeffding1994probability} that is defined by $\operatorname{Pr}\{|\hat{P}(s)-P(s)|\geq\epsilon\} \leq 2\rm e^{-2K\epsilon^2}$, we derive the desired sample complexity.
\end{proof}
When the compared MDPs are identical, the estimation error bound in Theorem \ref{Theorem:8} reduces to $2 \max_{s}d^{1\textnormal{-}\hat{1}}(s,s)$ for the single-MDP setting. We now prove the tightness of this new sampling error bound compared to the existing bound $\frac{2 \gamma}{1\!-\!\gamma}\max_{a,s} W_1( \hat{\mathbb{P}}(\cdot|s,a),\mathbb{P}(\cdot|s,a);d^\sim)$ in~\cite{10.1137/10080484X}. According to the transitive property of inequality on the Wasserstein distance defined in (\ref{preservation}), we have
\begin{align}
& \ d^{1\textnormal{-}\hat{1}}(s,s) \notag\\
=&\  H(X_{s},X_{s};\delta (d^{1\textnormal{-}\hat{1}}))  \notag\\
    \leq &\  \gamma W_1( \mathbb{P}(\cdot|s,a),\hat{\mathbb{P}}(\cdot|s,a) ;d^{1\textnormal{-}\hat{1}}) \notag\\
    \leq & \  \gamma W_1(\mathbb{P}(\cdot|s,a),\hat{\mathbb{P}}(\cdot|s,a), ;d^{1\textnormal{-}1})\notag\\
    &\ +\gamma W_1(\hat{\mathbb{P}}(\cdot|s,a),\hat{\mathbb{P}}(\cdot|s,a);d^{1\textnormal{-}\hat{1}})\notag\\
    \leq & \ \gamma W_1(\hat{\mathbb{P}}(\cdot|s,a), \mathbb{P}(\cdot|s,a) ;d^{1\textnormal{-}1})  + \gamma\max_{s}d^{1\textnormal{-}\hat{1}}(s,s).
\end{align}
Taking the maximum of both sides and rearranging the inequality yields 
\begin{align}
    2 \max_{s}d^{1\textnormal{-}\hat{1}}(s,s)\leq & \frac{2 \gamma}{{1-\gamma}}\max_{a,s} W_1( \hat{\mathbb{P}}(\cdot|s,a),\mathbb{P}(\cdot|s,a);d^\textnormal{1-1})\notag\\
    \leq& \frac{2 \gamma}{{1-\gamma}}\max_{a,s} W_1( \hat{\mathbb{P}}(\cdot|s,a),\mathbb{P}(\cdot|s,a);d^\sim).
\end{align}
We have thus established that our sampling error bound is strictly tighter than the BSM-derived bound given by \eqref{Eq:aggBSM} in~\cite{10.1137/10080484X}.

\subsection{Composite Applications}
The established metric properties of GBSM make it a well-suited tool for analyzing errors from different sources. Typically, in case the approximation combines both state aggregation and sampling-based estimation, where the approximated MDP is defined as $\mathcal{M}_{[\hat{1}]}=\langle [\mathcal{S}_1], \mathcal{A}_1, [\hat{\mathbb{P}}_1], R_1, \gamma\rangle$, we have
\begin{align}
    &\max_{s,s'}\big|d^{\textnormal{1-1}}(s,s')-d^{[\hat{1}]\textnormal{-}[\hat{1}]}(s,s')\big| \notag\\
    \leq &\  2 \max_s d^{1\textnormal{-}[\hat{1}]}(s,s) \notag\\
    \leq &\ 2 \max_s d^{\textnormal{1-[1]}}(s,s) + 2 \max_s d^{[1]\textnormal{-}[\hat{1}]}(s,s)
\end{align}
via the inter-MDP triangle inequality. It enables a decoupled analysis of error, and thus results in an explicit and closed-formed sample complexity, i.e., $K\geq-\ln(\alpha/2) \frac{\gamma^2\bar{R}^2|\mathcal{U}|^2}{2\epsilon^2(1-\gamma)^4}$ for an error below $\epsilon$ with probability of $1-\alpha$, where $\mathcal{U}$ is the set of representative states. In contrast, the sample complexity for estimation in BSM-based analysis~\cite{10.1137/10080484X} still depends on the unaggregated $d^\sim$, as shown in (\ref{Eq:aggBSM}), resulting in an asymptotic result.

Additionally, by combining Corollary~\ref{corolactionmapping} and Theorem~\ref{Theorem:8}, we establish the sample complexity for model-based RL to yield a near-optimal policy. In model-based RL, one first approximates an MDP from data, then derives a value function from the approximate MDP, and finally obtains a policy through the value function. We assume that in the approximated MDP model, the derived value function and associated policy are optimal, and have the following corollary.
\begin{corol}[\textbf{Sample complexity of model-based RL}]
Consider an RL task defined by an MDP $\mathcal{M}_1=\langle \mathcal{S}_1, \mathcal{A}_1, \mathbb{P}_1, R_1, \gamma\rangle$. To ensure a policy regret less than $\epsilon$ with a probability of $1-\alpha$, the sample complexity of model-based RL is given by
\begin{align}
   K \geq -\ln(\alpha/2) \frac{\gamma^2\bar{R}^2|\mathcal{S}_1|^2}{2\epsilon^2(1-\gamma)^6}
\end{align}
for each state-action pair.
\end{corol}
\begin{proof}
In model-based RL, where transitions are estimated within the original state and action spaces $\mathcal{S}_1$ and $\mathcal{A}_1$, the mappings $g(a)$ and $f(s')$ are implicitly assumed to be identity functions, i.e., $g(a)=a$ and $f(s')=s'$. By using Corollary~\ref{corolactionmapping}, we know that ensuring a policy regret less than $\epsilon$ (transferred from approximated MDP to the original one) requires $\max_s d^{1\textnormal{-}\hat{1}}(s,s)\leq \epsilon(1-\gamma) /2$. Note that in Theorem~\ref{Theorem:8}, we require $\max_s d^{1\textnormal{-}\hat{1}}(s,s)\leq \epsilon/2$. Replacing $\epsilon$ in (\ref{samplecomplexity}) by $\epsilon(1-\gamma)$ yields the desired result.
\end{proof}

\section{Extensions and Algorithm}
To demonstrate the compatibility and practicality of GBSM, we showcase how it can be extended to the variant of BSM and propose an efficient algorithm for computing GBSM in practical scenarios.
\subsection{Extensions to BSM variants}\label{laxgbsm}
The proposed GBSM framework is readily extendable to on-policy BSM~\cite{Castro_2020}, which computes state similarities between MDPs under non-optimal policies. To achieve this, we rewrite the Definition~\ref{DefiBSM} to
\begin{align}
    d^{1\textnormal{-}2}_\pi(s,s')\!=|R_1^\pi(s)\!-\!R_2^\pi(s')| \!+\!\gamma W_1(\mathbb{P}_1^\pi(\cdot|s), \mathbb{P}_2^\pi(\cdot|s') ;d^{1\textnormal{-}2}_\pi),
\end{align}
where $R_{1}^\pi(s)=\sum_a \pi(a|s)R_{1}(s,a)$ and $\mathbb{P}_1^\pi(\tilde{s}|s)=\sum_a \pi(a|s)\mathbb{P}_1(\tilde{s}|s,a),\forall\tilde{s}\in\mathcal{S}_1$ represent the expected reward and transition probabilities for a non-optimal policy $\pi$ in $\mathcal{M}_1$, with corresponding terms $R_{2}^\pi$ and $\mathbb{P}_2^\pi$ defined similarly for $\mathcal{M}_2$. Our theoretical properties are also preserved in this setting: the value difference bound in Theorem~\ref{Theorem:2} now applies to the on-policy value function by 
\begin{align}\label{VFBoundonPi}
    |V^\pi_1(s)-V^\pi_2(s')|\leq d_\pi^{1\textnormal{-}2}(s,s').
\end{align}
Then metric properties in Theorems~\ref{Theorem:3} and~\ref{Theorem:4} follow directly, and $d^{1\textnormal{-}2}_\pi(s,s')$ is bounded by an on-policy TV-based metric formulated as
\begin{align}
       d^{1\textnormal{-}2}_{\text{TV},\pi}(s,\!s')\!=&\Big\{\!\big|R_1^\pi(s)\!-\!R_2^\pi(s)\big|\!+\!\frac{\gamma \bar{R}}{1\!-\!\gamma} \text{TV}\big(\mathbb{P}_1^\pi(\cdot|s), \mathbb{P}_2^\pi(\cdot|s) \big)\!\Big\}
\end{align}
as Theorem~\ref{Theorem:5} for on-policy GBSM. In addition, as a direct consequence of (\ref{VFBoundonPi}), we have
\begin{align}
    \max_s |V^\pi_1(s)-V^\pi_{[1]}(s)|\leq \max_s d_\pi^{1\textnormal{-}[1]}(s,s)\leq \max_s \frac{\tilde{d}_\pi(s,[s])}{1-\gamma}
\end{align}
for VFA under non-optimal policy, a tighter bound compared with the existing result $\max_s 2\tilde{d}_\pi(s,[s])/(1-\gamma)$ in~\cite{10.5555/3540261.3540625}. See Appendix~\ref{Appendix6} for the proof.

\subsection{Efficient Algorithm for Calculating GBSM}
While the theoretical properties of GBSM provide a robust framework for multi-MDP analysis, its practical application requires an efficient computational method, particularly when dealing with large state spaces or when the exact transition probabilities and reward functions of an MDP are unknown (e.g., in a real-world system). To address these challenges, we combine the theorems on state aggregation and sampling-based estimation and propose an efficient algorithm for calculating GBSM.

Typically, when computing GBSM between two MDPs, there would be a fully-known source MDP $\mathcal{M}_\text{s}$ and an insufficiently known target MDP $\mathcal{M}_\text{t}$. For instance, in a sim-to-real task, the former one is the MDP of the simulation environment, while the latter one is the MDP in the real world, from which only a finite set of sampled data tuples $(s, a, s', r)$ is available. Our methodology, detailed in Algorithm~\ref{algoGBSM}, addresses this by first constructing an aggregated empirical model from the limited data and then aligning the state spaces in two compared MDPs to ensure a valid comparison. The process consists of three main stages:
\begin{algorithm}[t]
\caption{Computing GBSM in practice}\label{algoGBSM}
\begin{algorithmic}[1]
\State \textbf{Input:} 
Set of sampled data from target MDP $\mathcal{D}=\{(s, a, s', r)\}$, transition probability in source MDP $\mathbb{P}_\text{s}$, reward function in source MDP $R_\text{s}$, sample threshold $\eta_1$, and convergence threshold $\eta_2$.
\State \textbf{Output:} GBSM $d$
\State \textit{Stage 1: Construct representative state set for target MDP}
\State $\mathcal{U}_{\text{t}} \gets [~]$

\For{$s \in \mathcal{S}$}
    \For{$a \in \mathcal{A}$}
        \If{number of samples for $(s, a)$ in $\mathcal{D} \geq \eta_1$}
            \State Append $s$ to $\mathcal{U}_{\text{t}}$
        \EndIf
    \EndFor
\EndFor \\
$\mathcal{D}' = \{ (s, a, s', r) \in \mathcal{D} \mid s \in \mathcal{U}_{\text{t}} \land s' \in \mathcal{U}_{\text{t}} \}$
\For{$s \in \mathcal{U}_{\text{t}}$}
    \For{$a \in \mathcal{A}$}
        \State Estimate $\mathbb{P}_\text{t}(\cdot | s, a)$ and $R_\text{t}(s, a)$ from $\mathcal{D}'$
    \EndFor
\EndFor
\State \textit{Stage 2: Construct representative state set for source MDP}
\State $\mathcal{U}_{\text{s}} \gets \mathcal{U}_{\text{t}}$
\Repeat
    \For{$u \in \mathcal{U}_{\text{s}}$}
        \For{$a \in \mathcal{A}$}
            \For{$s \in \mathcal{S}$}
                \If{$\mathbb{P}_\text{s}(s | u, a) > 0$ and $s \notin \mathcal{U}_{\text{s}}$}
                    \State Append $s$ to $\mathcal{U}_{\text{s}}$
                \EndIf
            \EndFor
        \EndFor
    \EndFor
\Until{$|\mathcal{U}_{\text{s}}|$ does not increase}
\State \textit{Stage 3: Calculate GBSM iteratively}
\For{$s \in \mathcal{U}_{\text{t}},\ s' \in \mathcal{U}_{\text{s}}$}
    \State $d_0(s, s') \gets \max_{a \in \mathcal{A}} \left| R_\text{t}(s, a) - R_\text{s}(s', a) \right|$
\EndFor

\For{$n = 1, 2, \dots$}
    \State Compute $d_n$ according to Equation (5)
    \If{$\max_{s,s'} |d_n(s, s') - d_{n-1}(s, s')| \leq \eta_2$}
        \State ${d} \gets d_n$
        \State \Return ${d}$ 
    \EndIf
\EndFor

\end{algorithmic}
\end{algorithm}
\begin{enumerate}
    \item \textbf{Empirical model construction:} An empirical model of the target MDP is constructed from the collected dataset. To ensure the reliability of the estimated dynamics, we first form a representative state set $\mathcal{U}_\text{t}$, containing states for which a sufficient number of samples is available, and estimate empirical transition probabilities and rewards for representative states. 
    \item \textbf{Closed state space formulation:} A key prerequisite for the iterative calculation of GBSM is that the state spaces must be closed, meaning any state reachable from a state within the set must also be in the set. It is formalized by: $\forall (s,a)\in\mathcal{U}\!\times\!\mathcal{A}$, if $\mathbb{P}(s'|s,a) > 0$, then $ s'\in \mathcal{U}$. By filtering the collected data, we first construct a closed representative state space $\mathcal{U}_\text{t}$ in the real MDP. For the fully known source MDP $\mathcal{M}_\text{s}$, we begin by initializing the representative set as $\mathcal{U}_\text{s}\gets\mathcal{U}_\text{t}$. We then iteratively expand it by adding any states that are reachable from the current $\mathcal{U}_\text{s}$ until no new states can be added. This expansion ensures that $\mathcal{U}_\text{s}$ is a closed state space for the subsequent iterative computation.
    \item \textbf{Iterative GBSM computation:} With the empirical target MDP defined on $\mathcal{U}_\text{t}$ and the source MDP defined on the closed set $\mathcal{U}_\text{s}$, the GBSM is computed iteratively according to Equation (\ref{eq:dn}) until the distance converges within a predefined threshold.
\end{enumerate}

This algorithm facilitates the approximation of the GBSM by utilizing a carefully selected subset of representative states. It effectively addresses the challenges posed by the inaccessibility of exact transition probabilities and reward functions, which are common in practical applications.

\section{Numerical Results}

In this section, we empirically validate the theoretical results derived from GBSM. The section is organized into two parts. First, we use synthetic MDPs to rigorously verify the theoretical bounds and properties in a controlled setting. Second, we demonstrate the practicality of GBSM in a complex, real-world application involving sim-to-real RL for the resource allocation task in wireless network optimization.

\begin{figure*}[t]
\centering
\includegraphics[width=0.3\textwidth]{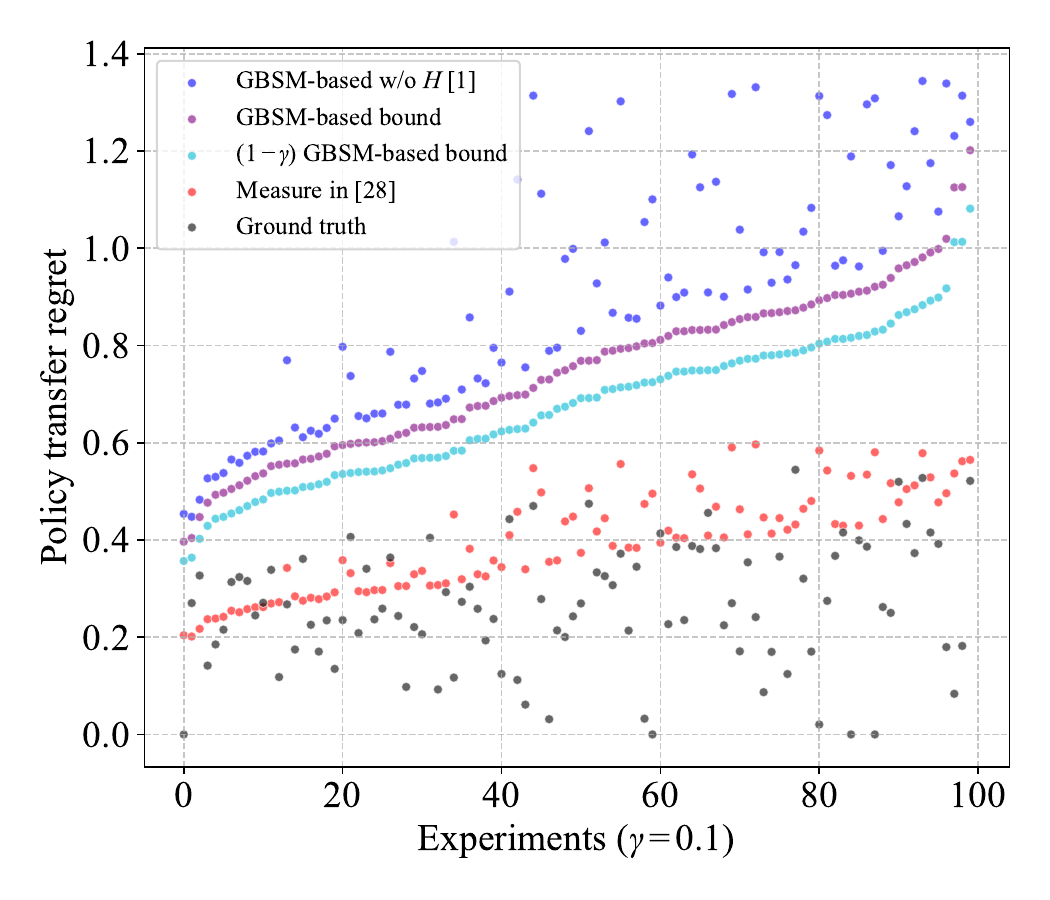}
\includegraphics[width=0.3\textwidth]{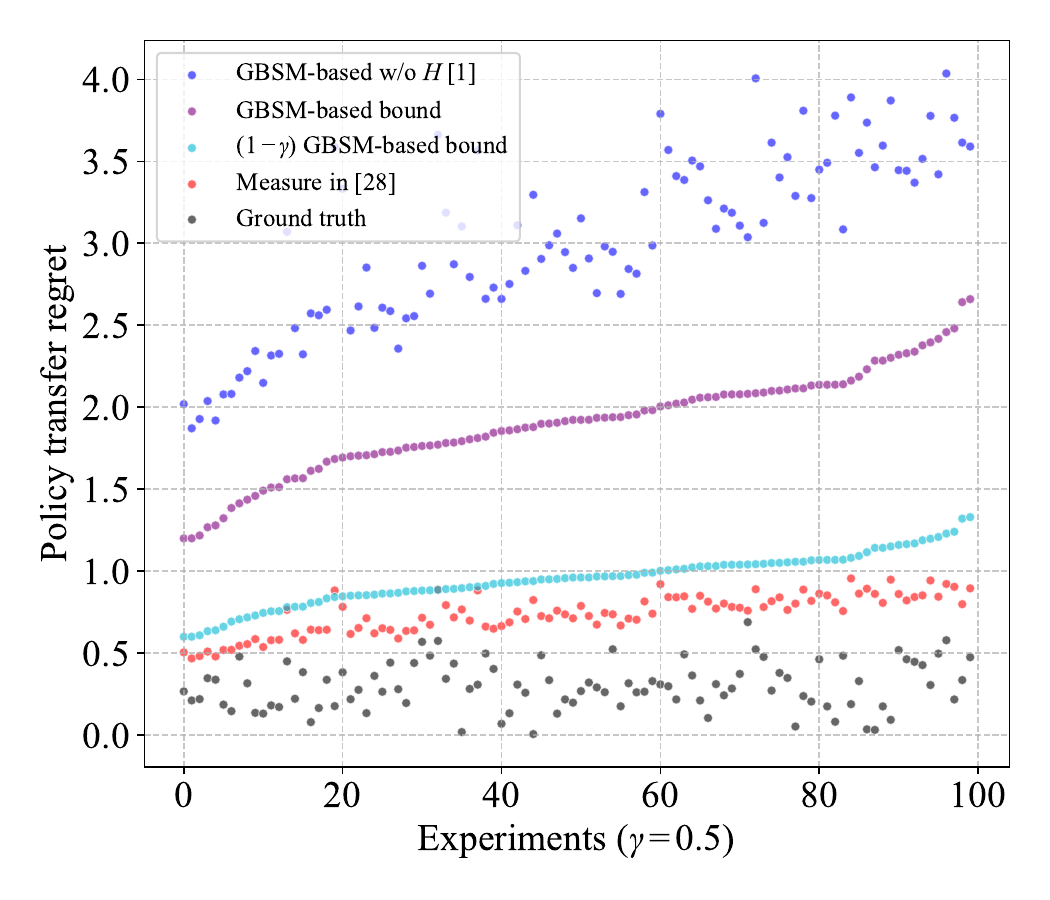}
\includegraphics[width=0.3\textwidth]{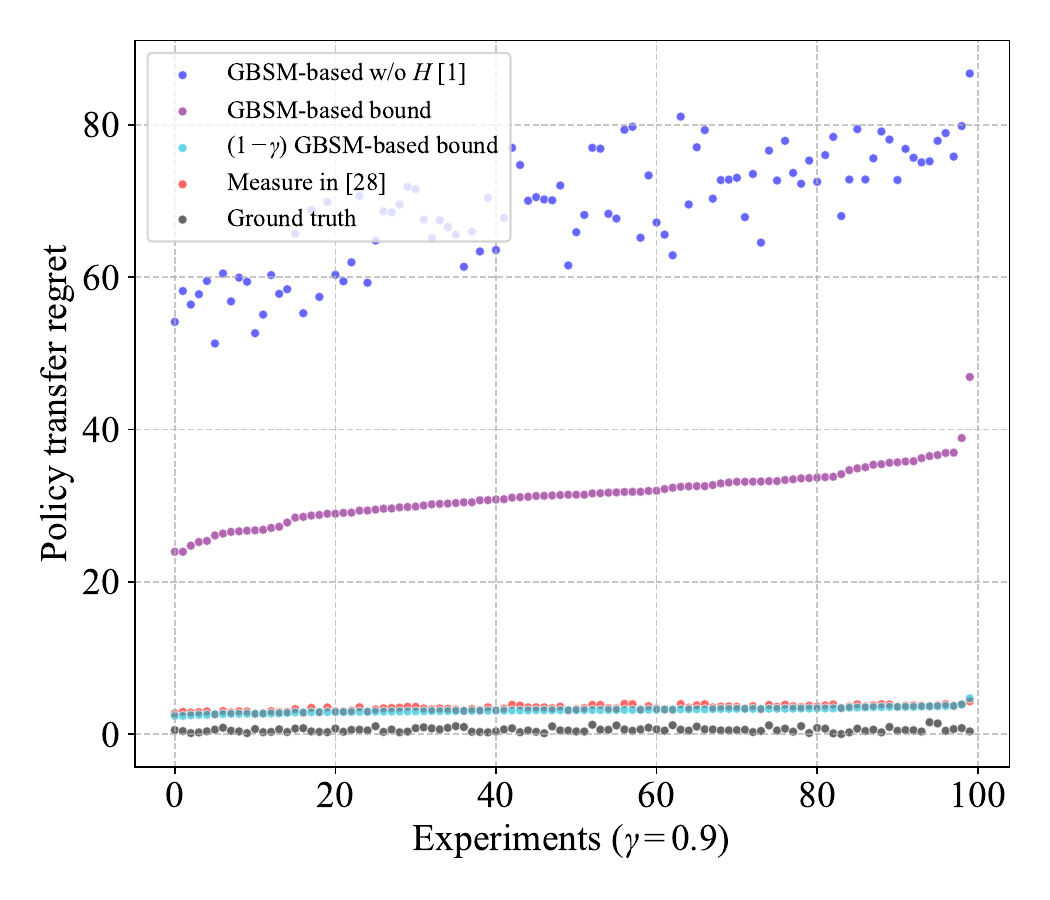}
\caption{Policy transfer experiments on random Garnet MDPs.}\label{Fig1}
\end{figure*}

\begin{figure*}[!t]
\centering
{\includegraphics[width=0.3\textwidth]{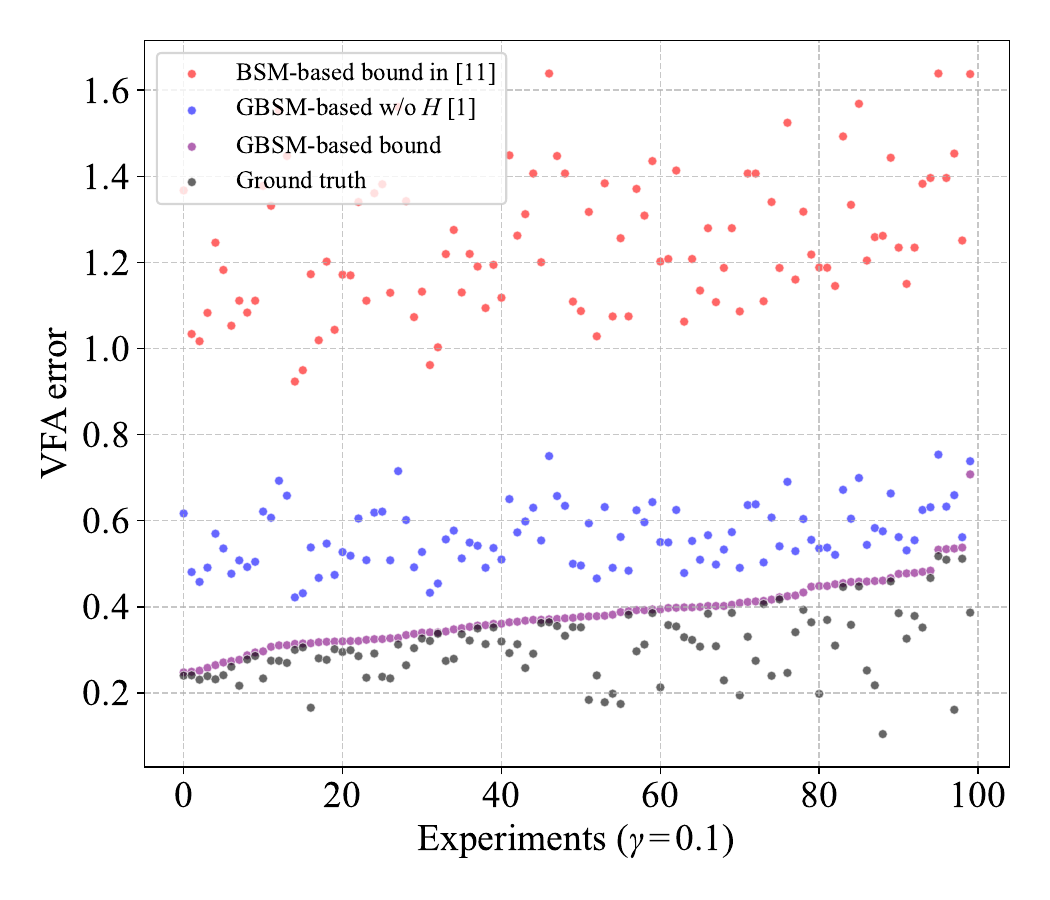}}
{\includegraphics[width=0.3\textwidth]{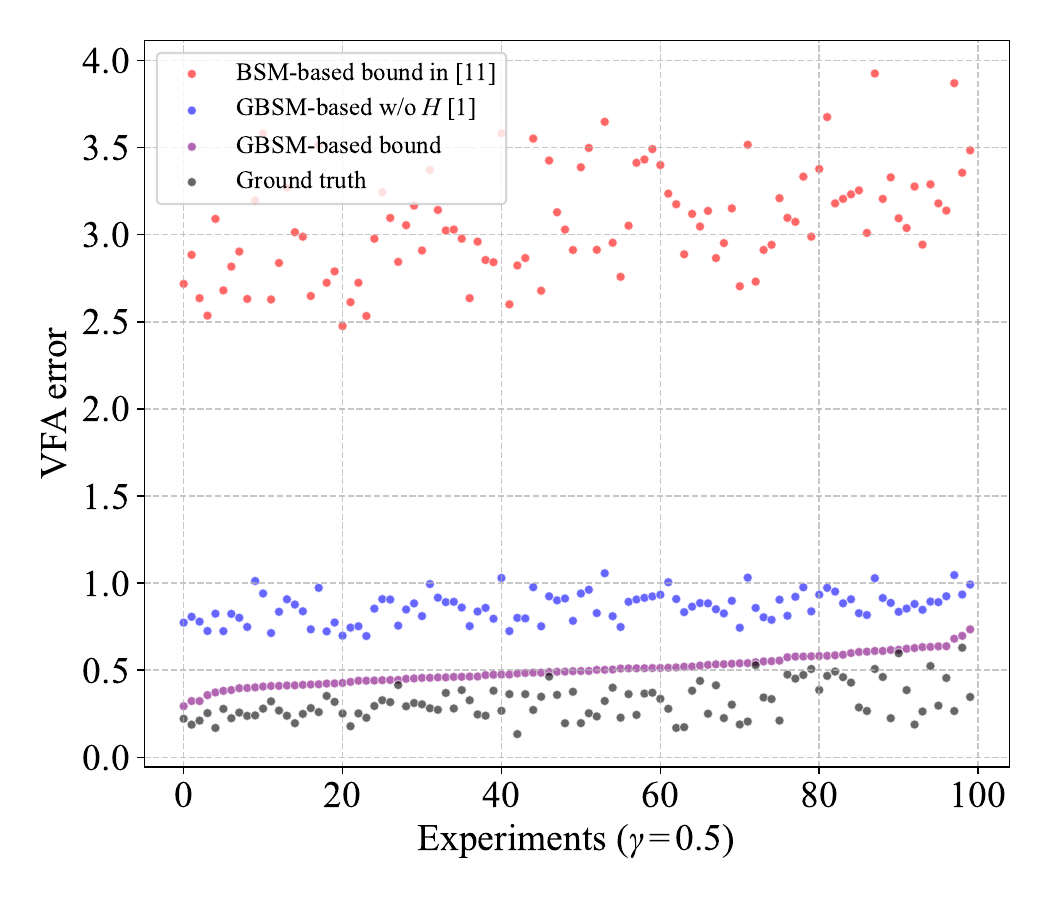}}
{\includegraphics[width=0.3\textwidth]{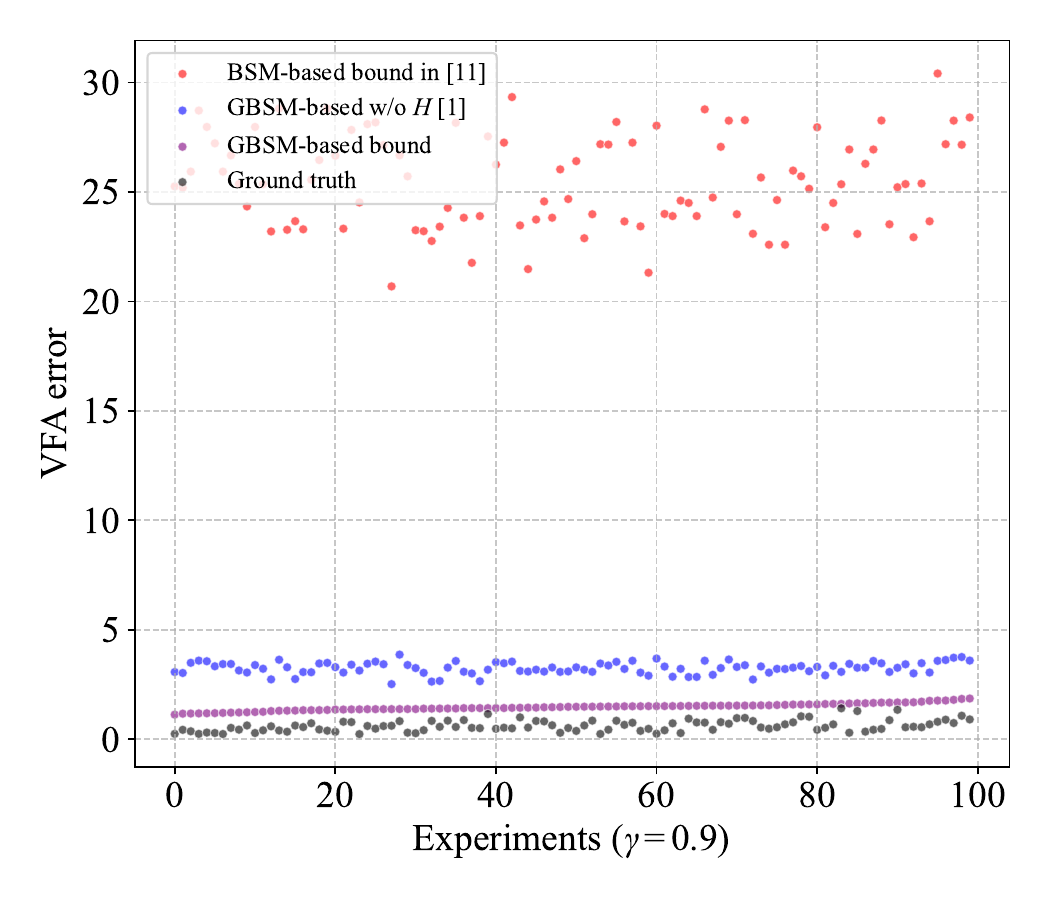}}
\caption{VFA experiments on random Garnet MDPs.}\label{Fig4}
\end{figure*}

\subsection{Validation on Synthetic MDPs}
To validate our theoretical findings, we first conduct experiments on random Garnet MDPs, a class of random MDP instances widely used for benchmarking \cite{tarbouriech2019active,grand2021scalable,sun2024policy}. To ensure comparability with BSM-based benchmarks, we fix the state and action spaces across all MDPs. Specifically, we construct these MDPs with $|\mathcal{S}| = 20, |\mathcal{A}| = 5$,  and a 50\% branching factor. We construct their aggregated and estimated counterparts to test the approximation bounds. In the aggregated MDPs, the reward functions and transition probabilities for half of the states are replaced by those of their representative states, while the estimated MDPs are established by introducing a Gaussian noise with a standard deviation ranging from 0.1 to 0.3 to the transition probabilities. Experiments were conducted for $\gamma\in\{0.1,0.5,0.9\}$, with 100 independent trials for each setting.


\subsubsection{Policy Transfer Regret Bound}
We first evaluate the policy transfer regret bound from Theorem~\ref{Theorem:6}. For each trial, we generate a source MDP $\mathcal{M}_1$ and a target MDP $\mathcal{M}_2$, compute the optimal policy $\pi_1^*$ in the source, and deploy it in the target. The ground truth regret is calculated as $\max_{s' \in \mathcal{S}_2} |V_2^*(s') - V_2^{\pi_1}(s')|$, where the precise value functions are computed under a tabular Q-learning setting~\cite{9207957}. We compute the GBSM-based bound using Corollary~\ref{corolactionmapping} given the identical action spaces, and compare it against the existing result of~\cite{10.5555/2936924.2936994} as well as our earlier conference work \cite{GBSM}, which defined the GBSM without the Hausdorff metric $H$. The x-axis represents the experiment index of 100 independent trials, and the y-axis plots the values of ground-truth error and (G)BSM-based bounds in each trial. As shown in Fig.~\ref{Fig1}, the GBSM-based bound consistently and rigorously upper-bounds the actual regret, whereas the existing empirical measure~\cite{10.5555/2936924.2936994} fails to do so. Furthermore, by incorporating the Hausdorff metric, the GBSM formulation proposed in this work yields a significantly tighter bound than our preliminary conference version~\cite{GBSM}, particularly in large $\gamma$ settings. The figure also visualizes a tighter empirical bound, $2\max d^{\textnormal{1-2}}$, derived by scaling the GBSM-based theoretical bound (\ref{boundwithoptimalactionmapping}) by a factor of $(1\!-\!\gamma)$. This empirical bound not only contains the ground-truth regret across all trials but also maintains its tightness, particularly for large values of $\gamma$. While we present this as an empirical observation pending formal proof, it could serve as valuable guidance for practical applications.


\subsubsection{Approximation Error Bounds}
We also validate the tightness of our GBSM-based bounds for VFA and SSA, comparing them against existing BSM-based ones. Specifically, the BSM-based VFA bound is computed via Theorem~\ref{Theorem:vfa} and compared with the BSM-based bound $2\tilde{\sigma}_1/(1-\gamma)$ in~\cite{DBLP:conf/iclr/0001MCGL21} as well as our earlier conference work \cite{GBSM}. We also compare them with the ground-truth error values to assess their tightness. For SSA bounds comparison, we employ the BSM-based bound in (\ref{Eq:aggBSM}). Since the estimation error bound in (\ref{Eq:aggBSM}) depends on the aggregation process, we decouple the two for clearer analysis. Specifically, the BSM-based aggregation SSA bound is given by $2\tilde{\sigma}_1 (2+\gamma)/(1-\gamma)$, and the estimation SSA bound is $\frac{2 \gamma}{1-\gamma}\max_{a,s} W_1(\hat{\mathbb{P}}(\cdot|s,a), \mathbb{P}(\cdot|s,a) ;d^{\sim})$~\cite{10.1137/10080484X}. The GBSM-based SSA bounds follow from Theorem \ref{Theorem:7} (aggregation) and Theorem \ref{Theorem:8} (estimation). As depicted in Fig.~\ref{Fig4} and Fig.~\ref{Fig2}, the GBSM-based bounds for VFA and SSA are consistently and significantly tighter than those derived from BSM across all tested $\gamma$ values. This corroborates our theoretical analysis and highlights the effectiveness of GBSM in multi-MDP analysis.


\begin{figure*}[!t]
\centering
{\includegraphics[width=0.3\textwidth]{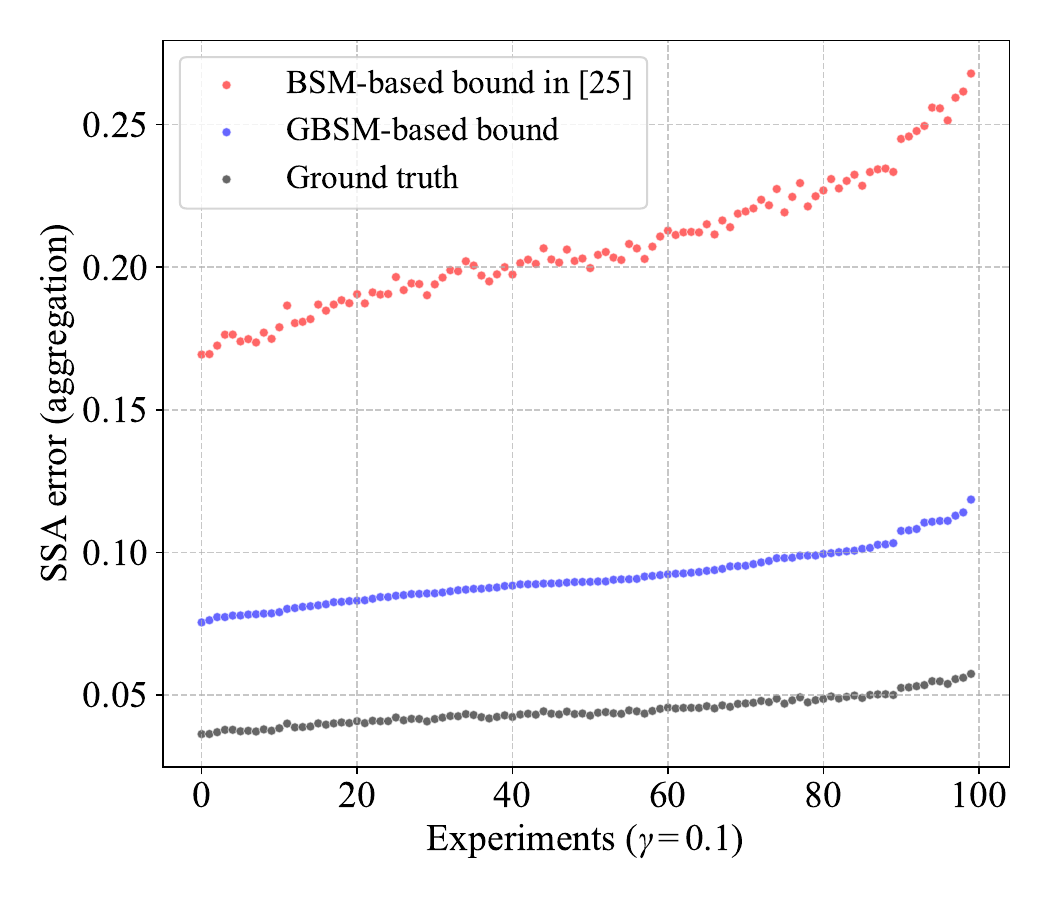}}
{\includegraphics[width=0.3\textwidth]{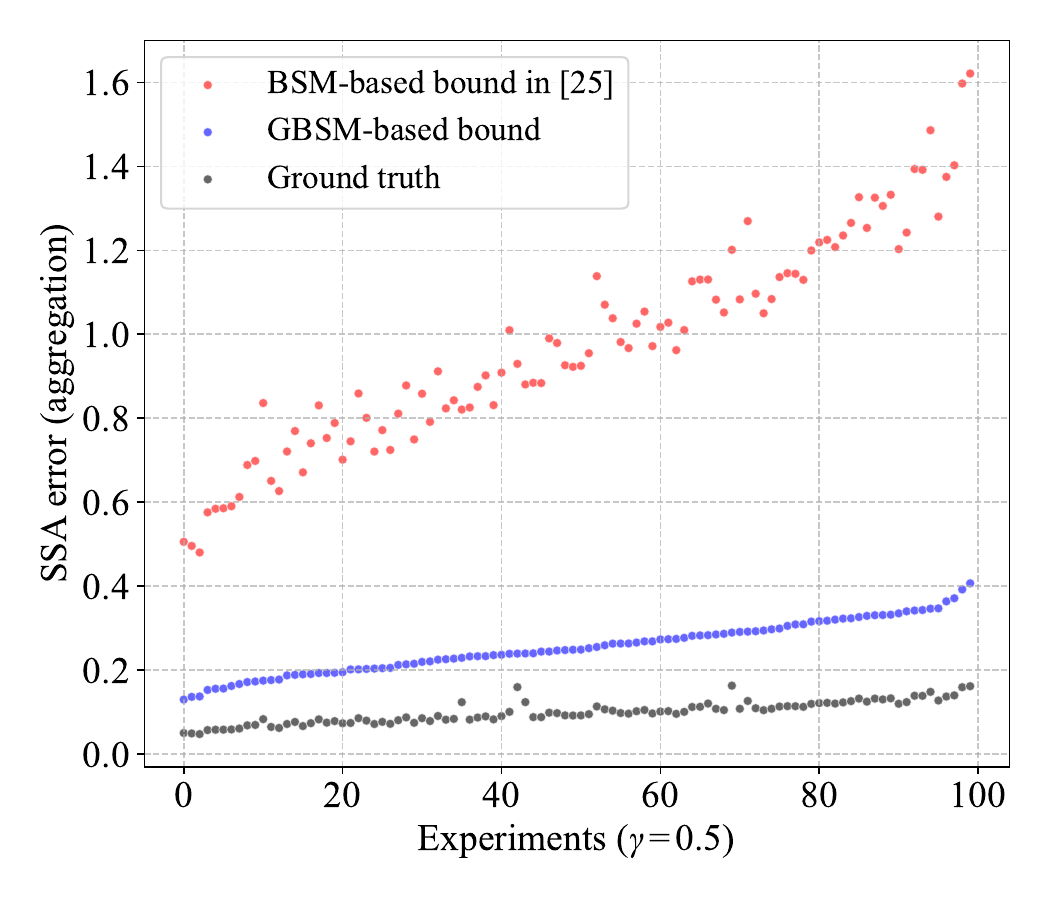}}
{\includegraphics[width=0.3\textwidth]{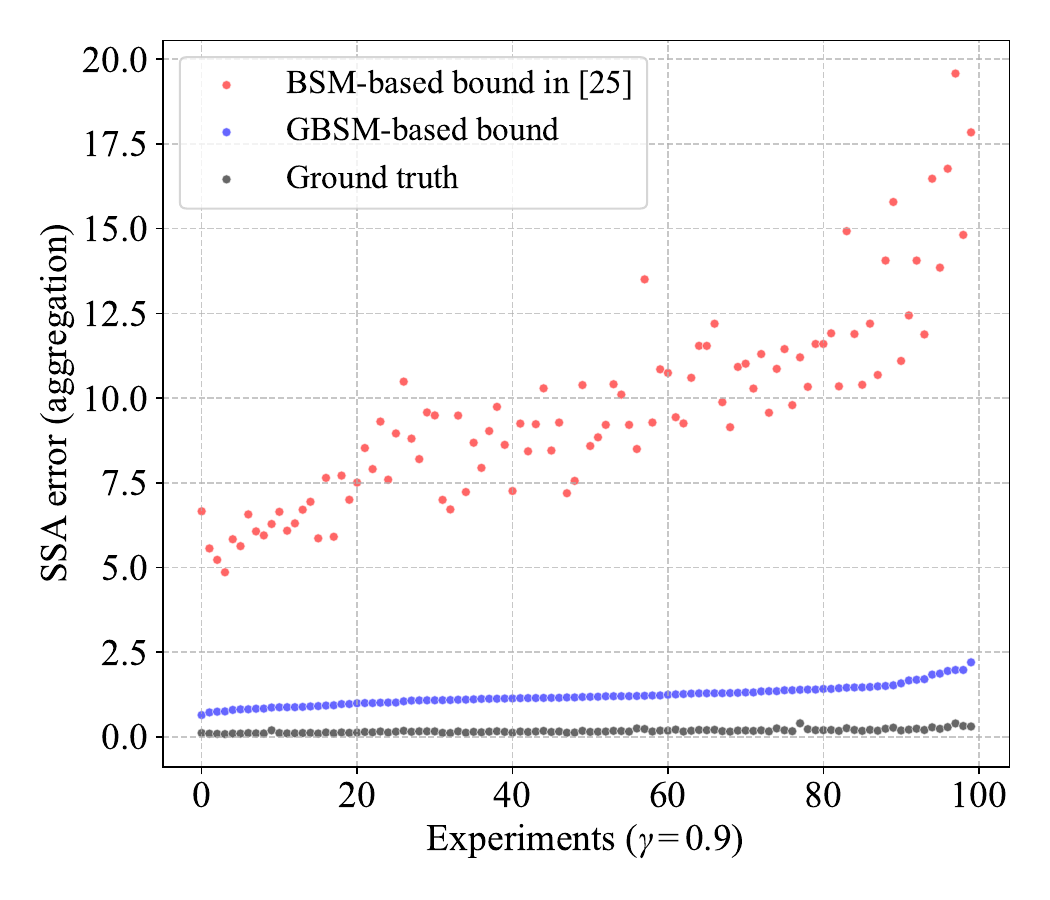}}

{\includegraphics[width=0.3\textwidth]{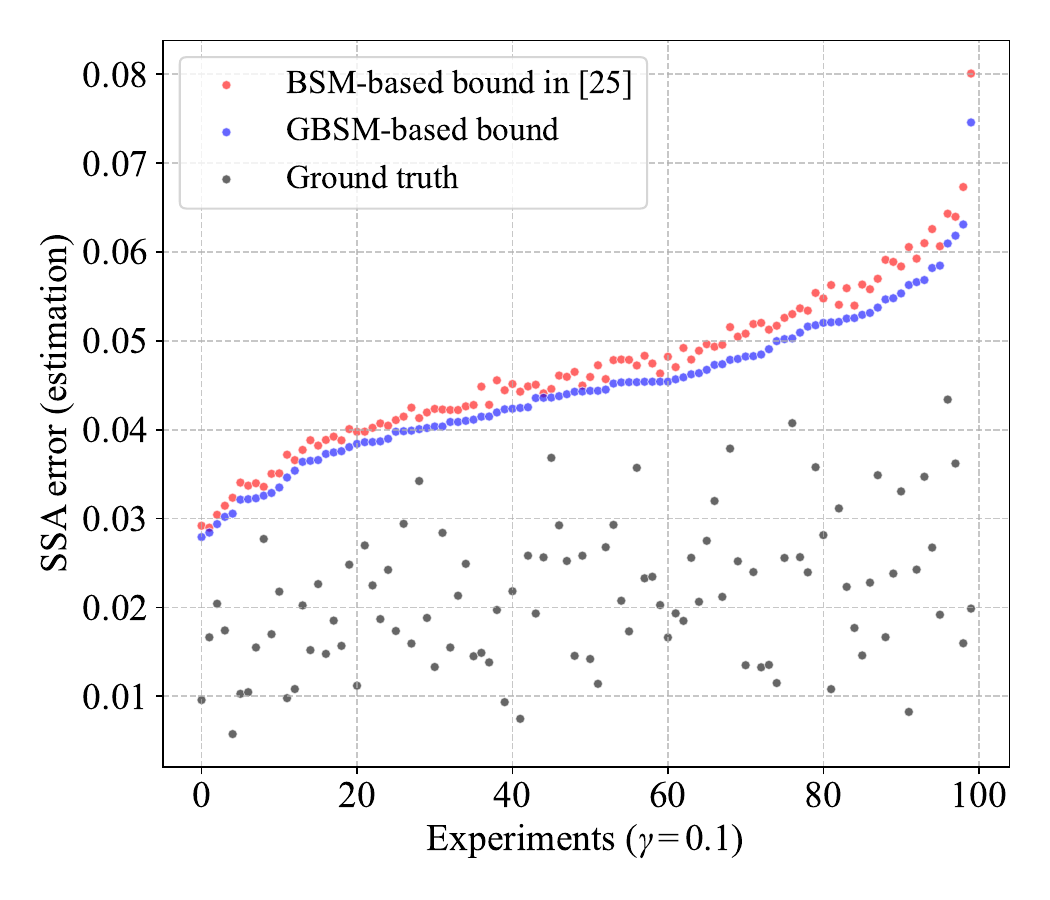}}
{\includegraphics[width=0.3\textwidth]{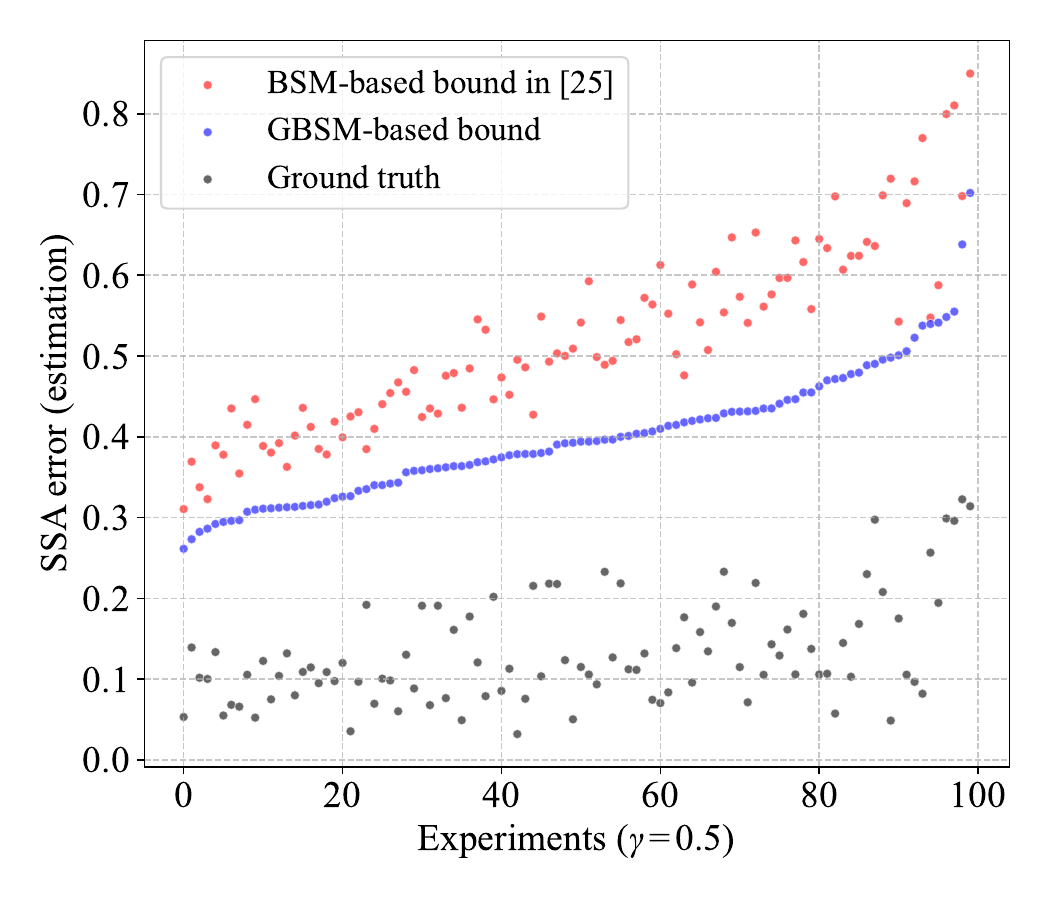}}
{\includegraphics[width=0.3\textwidth]{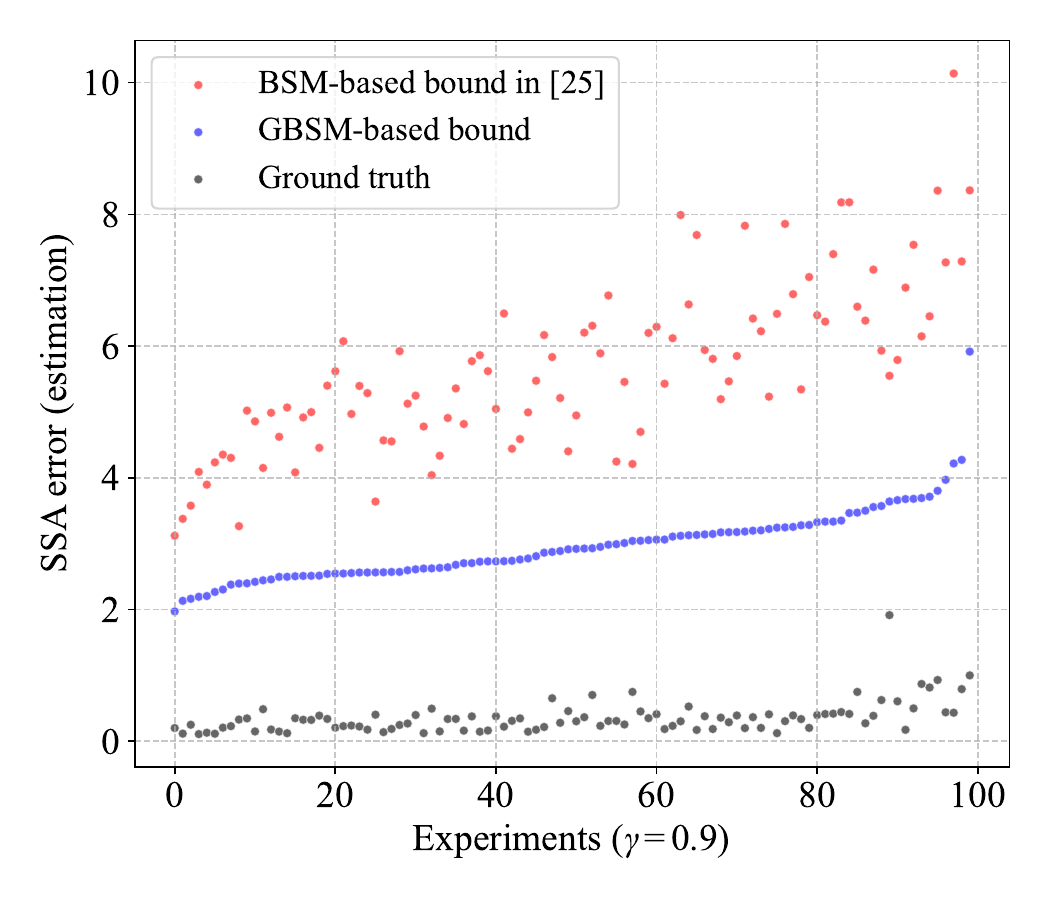}}
\caption{SSA experiments on random Garnet MDPs.}\label{Fig2}
\end{figure*}

\subsection{Real World Test: Sim-to-real RL in Wireless Networks}

\begin{figure}[t]
\centering
\subfloat[]{\includegraphics[width=0.35\textwidth]{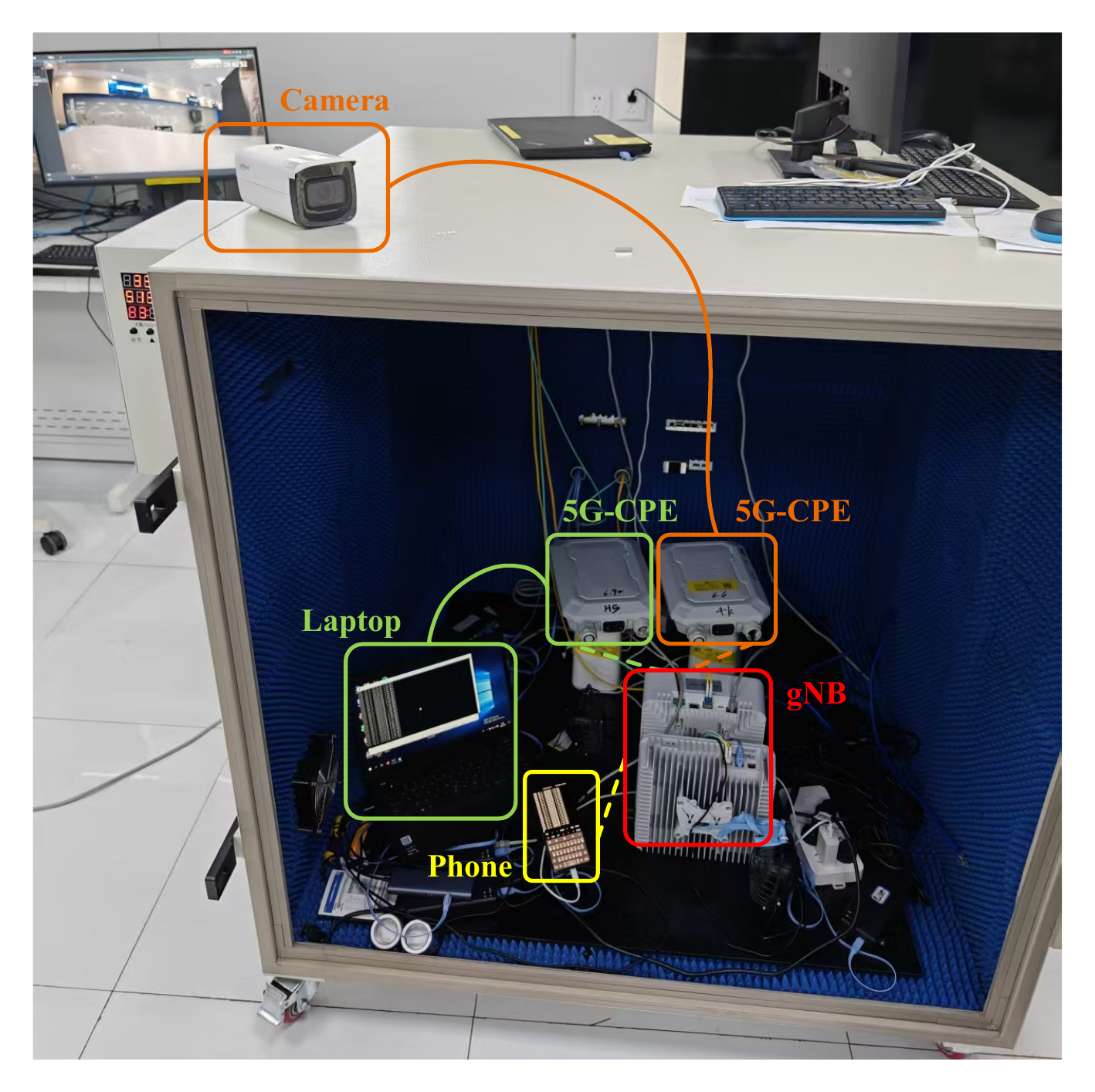}\label{Fig5.1}} \\
\subfloat[]{\includegraphics[width=0.5\textwidth]{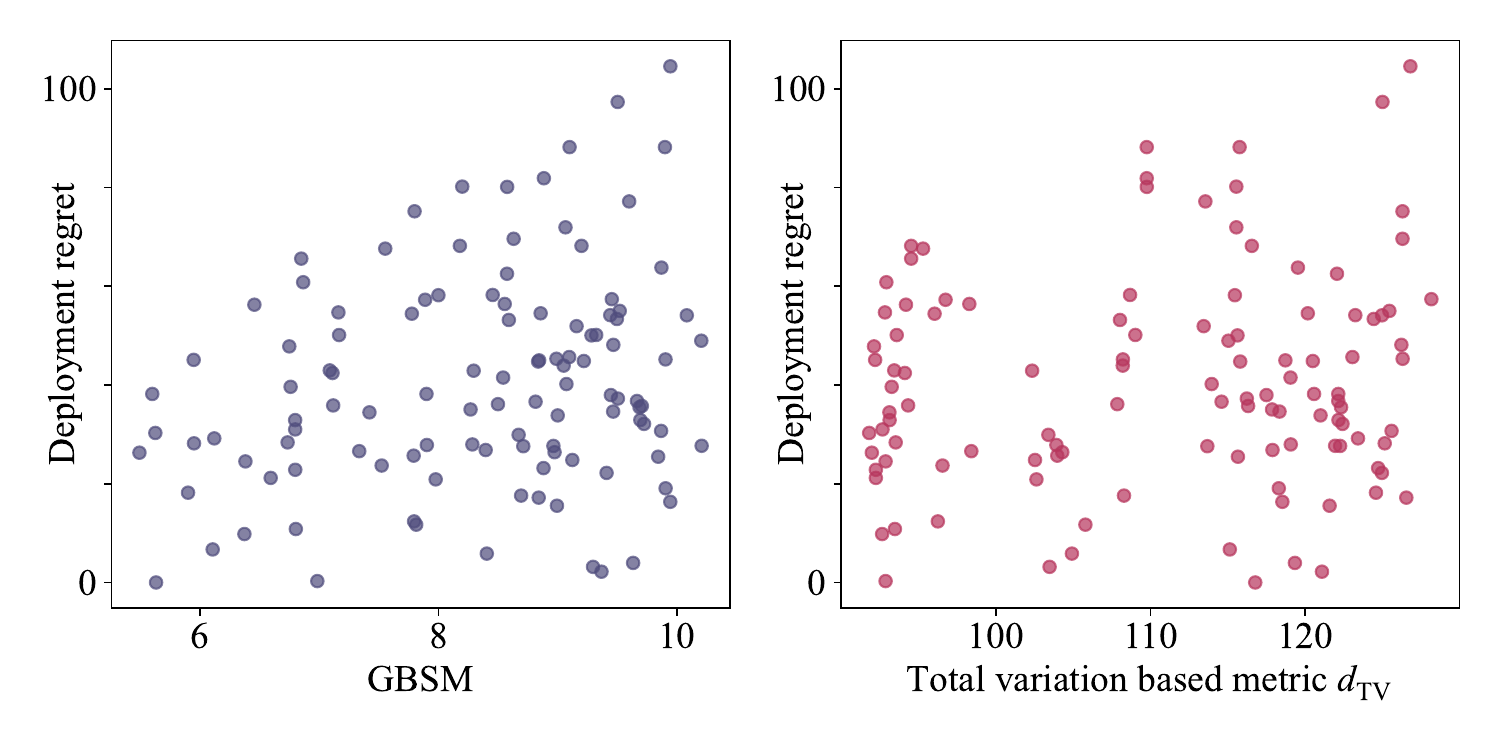}\label{Fig5.2}}
\caption{Sim-to-real experiment results in wireless network testbed.}
\end{figure}

To demonstrate GBSM's practicality, we apply it to enhance a sim-to-real RL paradigm for dynamic resource allocation in an uplink transmission scenario. The goal is to select the highest-fidelity simulation environments from a set of candidates to train an RL agent, thereby ensuring strong performance upon deployment in the real network. The experimental setup consists of a real-world wireless network testbed \cite{huang2021true,10183795} serving three users with heterogeneous traffic profiles (a camera, a gaming laptop, and a smartphone), as depicted in Fig.~\ref{Fig5.1}. We collected real-world operational data to generate hundreds of candidate simulation environments, each a learned MDP model of the real environment. For each simulation environment, we first calculated its GBSM with the real environment through Algorithm~\ref{algoGBSM} and then trained an RL agent within it to convergence. Policies from each simulation environment were finally deployed and evaluated in the real-world network testbed.

The results, shown in Fig.~\ref{Fig5.2}, reveal a strong, approximately linear correlation between the GBSM and the maximum policy regret during real-world deployment. In this context, regret is quantified by the sum-throughput loss compared to the best policy. This finding is consistent with both our theoretical analysis and synthetic experiments, confirming that simulation environments with lower GBSM values reliably yield policies with better real-world performance. These results also highlight the practical effectiveness of Algorithm~\ref{algoGBSM}, which successfully predicts policy transferability using a GBSM computed on only a subset of representative states, thereby validating our approximation for applications where exact MDPs are inaccessible. For comparison, we also evaluated the simpler metric based on the total variation distance, $d_\textnormal{TV}$. While it exhibited a similar positive trend, its discriminative ability was visibly weaker as it only captures immediate, one-step model discrepancies. In contrast, GBSM's recursive Bellman-like definition captures the long-term, cumulative impact of model inaccuracies, making it a far more reliable predictor of policy performance. Consequently, GBSM can be used to identify superior simulation environments before training, accelerating and reducing risk in the sim-to-real RL workflow.

\section{Conclusion}
In this paper, we have formally introduced GBSM and established its fundamental theoretical properties, including GBSM symmetry, inter-MDP triangle inequality, and distance bound on identical spaces. Leveraging these properties, we provide tighter bounds for policy transfer, state aggregation, and sampling-based estimation of MDPs, compared to the ones derived from BSM. To our knowledge, this is the first rigorous theoretical investigation of GBSM beyond simple definitional adaptation. We believe this work introduces a valuable theoretical tool for multi-MDP analysis. 

Under the new GBSM framework, there exist numerous potential applications for future study, and we conclude by highlighting several promising directions for future research.
\begin{itemize}
    \item \textit{Sim-to-real policy transfer:} GBSM can quantify the gap between simulated and real-world MDPs to predict policy transfer performance and guide the systematic refinement of simulation environments. Meanwhile, the approximation methods could be effectively employed to address the inaccessibility of precise transition probabilities and reward functions in the real world.
    \item \textit{Multi-task RL:} In multi-task settings, GBSM provides a principled metric for task similarity. This can inform task clustering, mitigate gradient interference issues, and coordinate policy optimization across different MDPs.
    \item \textit{Extension to average reward setting:} Our analysis is currently based on the discounted-reward criterion, and the theoretical bounds might become vacuous as the discount factor $\gamma\!\rightarrow\!1$. An important extension would be to develop a modified GBSM for the average-reward setting, which better reflects the RL objective in many long-horizon tasks.
\end{itemize}

\begin{appendices}

\section{Proof of Theorem \ref{Theorem:1}}\label{Appendix1}
To prove the existence of $d^{1\text{-}2}$, we introduce the Knaster-Tarski fixed-point theorem. Let $(\mathcal{X}, \preceq)$ denote a partial order, which means certain pairs of elements within the set $\mathcal{X}$ are comparable under the homogeneous relation $\preceq$~\cite{tarski1955lattice}. If this partial order has least upper bounds and greatest lower bounds for its arbitrary subsets, it is called a complete lattice. The Knaster-Tarski fixed-point theorem asserts that for a continuous function on a complete lattice, the iterative application of this function to the least element of the lattice converges to a fixed point $\bar{x}$, which satisfies $\bar{x}=f(\bar{x})$. Formally, the theorem is stated as follows.
\begin{lemma}[\textbf{Knaster-Tarski fixed-point theorem}~\cite{tarski1955lattice}]\label{lemma:fixedpoint}  
 If the partial order $(\mathcal{X}, \preceq)$ is a complete lattice and $f\!: \mathcal{X} \rightarrow \mathcal{X}$ is a continuous function. Then, $f$ has a least fixed point, given by 
 \begin{equation}
     \textnormal{fix}(f) = \sqcup_{n \in \mathbb{N}} f^{n}(x_0),
 \end{equation}
 where $x_0$ is the least element of $\mathcal{X}$, $\sqcup$ denotes the least upper bound, $f^{n}(x_0)=f(f^{(n-1)}(x_0))$, and $f^{(1)}(x_0)=f(x_0)$. Here, the continuity of $f$ is defined such that for any increasing sequence $\{x_n\}$ in $\mathcal{X}$, it satisfies
 \begin{equation}\label{eq:continu}
     f\left(\sqcup_{n \in \mathbb{N}}\left\{x_n\right\}\right)=\sqcup_{n \in \mathbb{N}}\left\{f\left(x_n\right)\right\}.
 \end{equation}
\end{lemma}
Let $\mathcal{D}$ denote the set of all cost functions, which are defined as maps that satisfy $\mathcal{S}_1 \times \mathcal{S}_2 \rightarrow [0,\frac{\bar{R}}{1-\gamma}]$. Equip $\mathcal{D}$ with the usual pointwise ordering: Consider two cost functions, say $d$ and $d'\in\mathcal{D}$, denote $d\leq d^{\prime}$ if and only if $d(s,s')\leq d^{\prime}(s,s')$ for any $s\in\mathcal{S}_1$ and $s'\in\mathcal{S}_2$. Then $\mathcal{D}$ forms a complete lattice with the least element $d^{1\text{-}2}_0$, i.e., the constant zero function. Given $s$ and $s'$, we regard the recursive definition in (\ref{eq:dn}) as a function of $d$ and accordingly define $F\!: \mathcal{D} \rightarrow \mathcal{D}$ by
\begin{align}
    F\left(d\,|\,s,s'\right)=H(X_s,X_{s'};\delta (d)).\label{defin:F}
\end{align}
Utilizing the Knaster-Tarski fixed-point theorem, the existence of $d^{1\text{-}2}$ is achieved if the continuity of $F$ holds on $\mathcal{D}$. 

We first prove the continuity of the second term in $F$. Define $F_{W_1}\!: \mathcal{D} \rightarrow \mathcal{D}$ by
\begin{equation}\label{defin:Fw}
\begin{aligned}
    F_{W_1}\left(d\,|\,s, s'\right)=&W_1\big(\mathbb{P}_1(\cdot|s,a), \mathbb{P}_2(\cdot|s',a);d \big).
\end{aligned}
\end{equation}
\begin{lemma} \label{lemma:conti:W1}
    $F_{W_1}$ is continuous on $\mathcal{D}$.
\end{lemma}
\begin{proof}
We follow the definition of continuity defined in Lemma~\ref{lemma:fixedpoint}. Let $s_i\in\mathcal{S}_1$ and $s_j\in\mathcal{S}_2$. Regard $F_{W_1}\left(s_i, s_j;d\right)$ as a function of $d$. Without loss of generality, we denote probability distributions $\{\mathbb{P}_1(\cdot|s_i,a), \mathbb{P}_2(\cdot|s_j,a)\}$ as $\{P,Q\}$ for brevity, and let $\rho\leq \rho '$, $\{\rho,\rho '\}\in\mathcal{D} $. Considering the optimal solution $\{\boldsymbol{\mu},\boldsymbol{\nu}\}$ for $W_1(P, Q;\rho)$ in the dual LP in (\ref{LP2}), we have
\begin{equation}
 \mu_i-\nu_j\leq \rho(s_i,s_j)\leq\rho'(s_i,s_j),\ \forall\;i,j,
\end{equation}
 which is derived from the pointwise ordering in $\mathcal{D}$.
Here, for the other $W_1(P, Q;\rho')$, $\{\boldsymbol{\mu},\boldsymbol{\nu}\}$ is a feasible, though not necessarily optimal, solution to the dual LP in (\ref{LP2}). Thus, we have
\begin{align}\label{lemma2ieq1}
\ W_1(P, Q;\rho)&= \ \sum_{i=1}^{|\mathcal{S}_1|} \mu_i P(s_i)- \sum_{j=1}^{|\mathcal{S}_2|}\nu_j Q(s_j) \notag\\
&\leq \ W_1(P, Q;\rho '),\ \forall\rho\leq \rho '.
\end{align}
By such a monotonicity, we have $W_1(P,Q;\rho)\leq W_1(P,Q;\sqcup_{n \in \mathbb{N}} \{\rho_n\}),\ \forall \rho \in \{\rho_n\}$ for any increasing sequence $\{\rho_n\}$ on $\mathcal{D}$. This further implies that $\sqcup_{n \in \mathbb{N}} \{W_1(P,Q;\rho_n)\}\leq W_1(P,Q;\sqcup_{n \in \mathbb{N}} \{\rho_n\})$.

We use the primal LP for the other side. Let $\boldsymbol{\lambda}^n$ denote the optimal solution in (\ref{LP1}) for $W_1(P,Q;\rho_n)$, which also satisfies the conditions for $W_1(P, Q;\sqcup_{n \in \mathbb{N}} \{\rho_n\})$. Define $\epsilon^n_{i,j} = \sqcup_{n \in \mathbb{N}} \{\rho_n\}(s_i,s_j)-\rho_n(s_i,s_j)$, then $\epsilon^n_{i,j} \geq 0 $ and $\lim_{n\rightarrow \infty}\epsilon^n_{i,j} = 0$ due to the monotonicity of the increasing sequence of $\{\rho_n\}$. Then, we have
\begin{align}\label{lemma2ieq2}
 &\ W_1(P,Q;\sqcup_{n \in \mathbb{N}}\{\rho_n\})\notag\\
\overset{(a)}{\leq}&\ \sum_{i=1}^{|\mathcal{S}_1|}\sum_{j=1}^{|\mathcal{S}_2|}\lambda^n_{i,j} \cdot \sqcup_{n \in \mathbb{N}}\{\rho_n\}(s_i,s_j) \notag\\
{=}&\  \sum_{i=1}^{|\mathcal{S}_1|}\sum_{j=1}^{|\mathcal{S}_2|} \lambda^n_{i,j} \rho_n(s_i,s_j)+\sum_{i=1}^{|\mathcal{S}_1|}\sum_{j=1}^{|\mathcal{S}_2|} \lambda^n_{i,j} \epsilon^{n}_{i,j}\notag\\
{=}&\ W_1(P,Q;\rho_n) + \sum_{i,j=1}^{|\mathcal{S}|} \lambda^n_{i,j} \epsilon^{n}_{i,j}\notag\\
{\leq}&\ \sqcup_{n \in \mathbb{N}} \{W_1(P,Q;\rho_n)\} +\max_{i,j} \{{\epsilon}^{n}_{i,j}\}.
\end{align}
Here, step~$(a)$ follows from the fact that $\boldsymbol{\lambda}^n$ is the optimal solution for $W_1(P,Q;\rho_n)$ rather than $W_1(P,Q;\sqcup_{n \in \mathbb{N}}\{\rho_n\})$.
Taking $n\rightarrow \infty$, we have $\sqcup_{n \in \mathbb{N}} \{W_1(P,Q;\rho_n)\}\geq W_1(P,Q;\sqcup_{n \in \mathbb{N}} \{\rho_n\})$. Following from the above two inequalities from both directions, it is readily to get $\sqcup_{n \in \mathbb{N}} \{W_1(P,Q;\rho_n)\}= W_1(P,Q;\sqcup_{n \in \mathbb{N}} \{\rho_n\})$. Thus, for any $i$ and $j$,
\begin{align}
 &\ F_{W_1}\big(\sqcup_{n \in \mathbb{N}}\{\rho_n\}\,|\,s_i, s_j\big) \notag\\
=&\ W_1\big(\mathbb{P}_1(\cdot|s_i,a), \mathbb{P}_2(\cdot|s_j,a);\sqcup_{n \in \mathbb{N}}\{\rho_n\}\big)\notag\\
=&\ \sqcup_{n \in \mathbb{N}} \Big\{W_1(\mathbb{P}_1(\cdot|s_i,a), \mathbb{P}_2(\cdot|s_j,a);\rho_n)\Big\}\notag\\
=&\ \sqcup_{n \in \mathbb{N}}\Big\{ F_{W_1}\left(\rho_n\,|\,s_i, s_j\right)\Big\}.
\end{align}
Now that the continuity of $F_{W_1}$ in (\ref{defin:Fw}) on $\mathcal{D}$ is established.
\end{proof}
\noindent Armed with Lemma~\ref{lemma:conti:W1}, we are ready to establish the continuity of $F$ as follows.
\begin{lemma}\label{lemma:F}
    $F$ is continuous on $\mathcal{D}$.
\end{lemma}
\begin{proof}
Considering an arbitrary increasing sequence $\{\rho_n\}$ on $\mathcal{D}$, for any $i$ and $j$, we have
\begin{align}
 &\ F(\sqcup_{n\in\mathbb{N}}\{\rho_n\}\,|\,s_i,s_j) \notag\\
 {=}&\ H(X_{s_i},X_{s_j};\delta (\sqcup_{n\in\mathbb{N}}\{\rho_n\}))\notag\\
{=}&\ \max\Big\{\max_{a\in\mathcal{A}_1} \min_{a'\in\mathcal{A}_2}\delta(\sqcup_{n\in\mathbb{N}}\{\rho_n\})((s_i,a),(s_j,a')),\notag\\
&\qquad\quad \min_{a\in\mathcal{A}_1}\max_{a'\in\mathcal{A}_2}\delta(\sqcup_{n\in\mathbb{N}}\{\rho_n\})((s_i,a),(s_j,a'))\Big\}
 \end{align}
 Without loss of generality, assume that $\max_{a\in\mathcal{A}_1} \min_{a'\in\mathcal{A}_2}\delta(d)((s_i,a),(s_j,a'))$ is greater, then
 \begin{align}
  &\ F(\sqcup_{n\in\mathbb{N}}\{\rho_n\}\,|\,s_i,s_j) \notag\\
{=}&\ \max_{a\in\mathcal{A}_1}\min_{a'\in\mathcal{A}_2} \Big\{\left|R_1(s_i,a)-R_2(s_j,a')\right|\notag\\
    &\qquad\quad+\gamma W_1\left(\mathbb{P}_1(\cdot|s_i,a), \mathbb{P}_2(\cdot|s_j,a') ;\sqcup_{n\in\mathbb{N}}\{\rho_n\}\right)\Big\}\notag\\
{=}&\ \max_{a\in\mathcal{A}_1}\min_{a'\in\mathcal{A}_2} \Big\{\left|R_1(s_i,a)-R_2(s_j,a')\right|\notag\\
    &\qquad\quad + \gamma \sqcup_{n\in\mathbb{N}} \left\{W_1\left(\mathbb{P}_1(\cdot|s_i,a), \mathbb{P}_2(\cdot|s_j,a') ;\rho_n\right)\right\}\Big\} \notag\\
{=}&\ \sqcup_{n\in\mathbb{N}}\Big\{\max_{a\in\mathcal{A}_1}\min_{a'\in\mathcal{A}_2}\Big\{\left|R_1(s_i,a)-R_2(s_j,a')\right|\notag\\
    &\ \qquad\quad +\gamma W_1\left(\mathbb{P}_1(\cdot|s_i,a), \mathbb{P}_2(\cdot|s_j,a') ;\rho_n\right)\Big\}\Big\} \notag\\
{\leq}&\ \sqcup_{n\in\mathbb{N}}\{F(\rho_n\,|\,s_i,s_j)\}.
 \end{align}
\end{proof}
Now that the existence of $d^{1\text{-}2}$ is established by using Lemma~\ref{lemma:fixedpoint} and Lemma~\ref{lemma:F}.

\section{Proof of Theorem \ref{Theorem:2}}\label{Appendix2}
This proves the optimal value difference bound between MDPs by induction.
For the base case, we have
        \begin{align}\label{corol1-1}
            |V_1^{0} (s_i)-V_2^{0} (s_j)| = d_\textnormal{0}^{1\textnormal{-}2}(s_i,s_j)=0,\forall (s_i,s_j)\in \mathcal{S}_1\times\mathcal{S}_2.
        \end{align}
    By the induction hypothesis, we assume that for an arbitrary $n$,
    \begin{align}\label{vv_ineq}
        |V_1^{n} (s_i)-V_2^{n} (s_j)|\leq&\ d^{1\text{-}2}_n(s_i,s_j),\ \forall\ (s_i,s_j)\in \mathcal{S}_1\times\mathcal{S}_2.
    \end{align}
    Without loss of generality, assume that $V_1^{n+1} (s_i)\geq V_2^{n+1} (s_j)$, and the induction follows
\begin{align}\label{corol1-3}
           &\ \left|V_1^{n+1} (s_i)-V_2^{n+1} (s_j)\right| \notag\\
           {=} &\ \left|\max_{a\in\mathcal{A}_1}\left\{R_1(s_i,a)+\gamma\sum_{k=1}^{|\mathcal{S}_1|}\mathbb{P}_1(s_k|s_i,a)V_1^{n}(s_k)\right\} \right.\notag\\
    &\qquad\left.-\max_{a'\in\mathcal{A}_2}\left\{R_2(s_j,a')+\gamma\sum_{k=1}^{|\mathcal{S}_2|}\mathbb{P}_2(s_k|s_j,a')V_2^{n}(s_k)\right\}\right|\notag\\
     \overset{(a)}{=} &\ \left|\left\{R_1(s_i,a_{\max})+\gamma\sum_{k=1}^{|\mathcal{S}_1|}\mathbb{P}_1(s_k|s_i,a_{\max})V_1^{n}(s_k)\right\} \right.\notag\\
    &\qquad\left.-\left\{R_2(s_j,a'_{\max})+\gamma\sum_{k=1}^{|\mathcal{S}_2|}\mathbb{P}_2(s_k|s_j,a'_{\max})V_2^{n}(s_k)\right\}\right|\notag\\
           =&\ \min_{a'\in\mathcal{A}_2} \left|\left\{R_1(s_i,a_{\max})+\gamma\sum_{k=1}^{|\mathcal{S}_1|}\mathbb{P}_1(s_k|s_i,a_{\max})V_1^{n}(s_k)\right\} \right.\notag\\
    &\qquad\left.-\left\{R_2(s_j,a')+\gamma\sum_{k=1}^{|\mathcal{S}_2|}\mathbb{P}_2(s_k|s_j,a')V_2^{n}(s_k)\right\}\right|\notag\\
               \leq&\ \max_{a\in\mathcal{A}_1} \min_{a'\in\mathcal{A}_2} \left|\left\{R_1(s_i,a)+\gamma\sum_{k=1}^{|\mathcal{S}_1|}\mathbb{P}_1(s_k|s_i,a) V_1^{n}(s_k)\right\} \right.\notag\\
    &\qquad\qquad\left.-\left\{R_2(s_j,a')+\gamma\sum_{k=1}^{|\mathcal{S}_2|}\mathbb{P}_2(s_k|s_j,a')V_2^{n}(s_k)\right\}\right|\notag\\
           {\leq} &\ \max_{a\in\mathcal{A}_1} \min_{a'\in\mathcal{A}_2} \Bigg\{\bigg|R_1(s_i,a) - R_2 (s_j,a')\bigg| \notag\\
           &\ +\gamma\bigg|\sum_{k=1}^{|\mathcal{S}_1|} \mathbb{P}_1 (s_k|s_i,a) V_1^{n}(s_k)- \sum_{k=1}^{|\mathcal{S}_2|} \mathbb{P}_2(s_k|s_j,a') V_2^{n}(s_k)\bigg|\Bigg\}\notag\\
           \overset{(b)}{\leq}&\ \max_{a\in\mathcal{A}_1} \min_{a'\in\mathcal{A}_2} \bigg\{\left|R_1(s_i,a)- R_2 (s_j,a')\right| \notag\\
           &\qquad\qquad +\gamma W_1\left(\mathbb{P}_1(\cdot|s_i,a),\mathbb{P}_2(\cdot|s_j,a');d^{1\text{-}2}_n\right)\bigg\}\notag\\
           \leq &\ d^{1\text{-}2}_\textnormal{n+1}(s_i,s_j),\ \forall\ (s_i,s_j)\in \mathcal{S}_1\times\mathcal{S}_2.
        \end{align}
    Here, $a_{\max}$ and $a'_{\max}$ in step~$(a)$ denotes actions for achiving optimal value functions $V_1^{n+1}$ and $V_2^{n+1}$, respectively. Step~$(b)$ follows from the fact that $\big( V_1^{n}(s_k) \big)_{k=1}^{|\mathcal{S}_1|}$ and $\big( V_2^{n}(s_k) \big)_{k=1}^{|\mathcal{S}_2|}$ form a feasible, but not necessarily the optimal, solution to the dual LP in (\ref{LP2}) for $W_1\big(\mathbb{P}_1(\cdot|s_i,a),\mathbb{P}_2(\cdot|s_j,a);d^{1\text{-}2}_n\big)$.

   Now from (\ref{corol1-1})-(\ref{corol1-3}), we have $|V_1^{n}(s)-V_2^{n}(s')|\leq d^{1\text{-}2}_{n}(s,s'),\ \forall\ (s,s')\in \mathcal{S}_1\times\mathcal{S}_2,\ \forall n\in\mathbb{N}$. Taking $n\rightarrow \infty$ yields the desired result.

\section{Proof of Theorem \ref{Theorem:4}}\label{Appendix:Theorem:4}
To prove the inter-MDP triangle inequality, we first need to prove the transitive property of inequality on the Wasserstein distance, that is, the Wasserstein distance between the three distributions follows $W_1(\mathbb{P}^1,\mathbb{P}^2 ;d^\textnormal{1-2}) \leq  W_1(\mathbb{P}^1,\mathbb{P}^\textnormal{3};d^{\textnormal{1-3}} ) +W_1(\mathbb{P}^\textnormal{3},\mathbb{P}^2;d^{\textnormal{3-2}} )$ if (\ref{triangleieq}) holds, where $\mathbb{P}^1$, $\mathbb{P}^2$, and $\mathbb{P}^3$ denote arbitrary distributions on $\mathcal{S}_1$, $\mathcal{S}_2$, and $\mathcal{S}_3$.

Let $(s_i,s_j,s_k)\!\in\! \mathcal{S}_1\times\mathcal{S}_2\times\mathcal{S}_3$. Define $\boldsymbol{\lambda}^{1,3}$ as the optimal transportation plan for $W_1(\mathbb{P}^1,\mathbb{P}^3 ;d^\textnormal{1-3})$ in primal LP (\ref{LP1}), with elements $ \lambda^{1,3}_{i,k}$ satisfying $ \sum_{k=1}^{|\mathcal{S}_3|} \lambda^{1,3}_{i,k}=\mathbb{P}^1(s_i)$ and $\sum_{i=1}^{|\mathcal{S}_1|}\lambda^{1,3}_{i,k}=\mathbb{P}^3(s_k)$. Similarly define $\boldsymbol{\lambda}^{3,2}$ for $W_1(\mathbb{P}^3,\mathbb{P}^2 ;d^\textnormal{3-2})$ with elements $\lambda^{3,2}_{k,j}$. Construct $\boldsymbol{\lambda}^{1,3,2}$ with elements $\lambda^{1,3,2}_{i,k,j}$ satisfying $\sum_{j=1}^{|\mathcal{S}_2|} \lambda^{1,3,2}_{i,k,j} = \lambda^{1,3}_{i,k}$ and $\sum_{i=1}^{|\mathcal{S}_1|} \lambda^{1,3,2}_{i,k,j} = \lambda^{3,2}_{k,j}$.
    Such a $\boldsymbol{\lambda}^{1,3,2}$ does exist according to the Gluing Lemma in~\cite{villani2009optimal}. Then, note that
    \begin{align}
        \sum_{j=1}^{|\mathcal{S}_2|} \sum_{k=1}^{|\mathcal{S}_3|}\lambda^{1,3,2}_{i,k,j} &= \sum_{k=1}^{|\mathcal{S}_3|} \lambda^{1,3}_{i,k} = \mathbb{P}^1(s_i),\notag\\ 
        \sum_{i=1}^{|\mathcal{S}_1|} \sum_{k=1}^{|\mathcal{S}_3|} \lambda^{1,3,2}_{i,k,j} &= \sum_{k=1}^{|\mathcal{S}_3|} \lambda^{3,2}_{k,j} = 
        \mathbb{P}^2(s_j), 
    \end{align}
    thus $\sum_{k=1}^{|\mathcal{S}_3|} \boldsymbol{\lambda}^{1,3,2}$ is a feasible, but not necessarily the optimal, solution to the primal LP in (\ref{LP1}) for $W_1(\mathbb{P}^1,\mathbb{P}^2 ;d^\textnormal{1-2})$. Consequently, we have
     \begin{align}\label{preservation}
     &\ W_1\big(\mathbb{P}^1,\mathbb{P}^2 ;d^\textnormal{1-2}\big)\notag\\
     {\leq}&\  \sum_{i=1}^{|\mathcal{S}_1|} \sum_{j=1}^{|\mathcal{S}_2|} \Big( \sum_{k=1}^{|\mathcal{S}_3|} \lambda^{1,3,2}_{i,k,j} \Big) d^{\textnormal{1-2}}(s_i, s_j)\notag\\
     \overset{(a)}{\leq}&\  \sum_{i=1}^{|\mathcal{S}_1|} \sum_{j=1}^{|\mathcal{S}_2|} \Big( \sum_{k=1}^{|\mathcal{S}_3|} \lambda^{1,3,2}_{i,k,j} \Big)
     \Big(d^{\textnormal{1-3}}(s_i,s_k)+d^{\textnormal{3-2}}(s_k,s_j)\Big)\notag\\
     {=}\; &\  \sum_{i=1}^{|\mathcal{S}_1|} \sum_{k=1}^{|\mathcal{S}_3|} \lambda^{1,3}_{i,k} d^{\textnormal{1-3}}(s_i,s_k) + \sum_{k=1}^{|\mathcal{S}_3|}\sum_{j=1}^{|\mathcal{S}_2|}  \lambda^{3,2}_{k,j} d^{\textnormal{3-2}}(s_k,s_j) \notag\\
     {=}\; &\ W_1\big(\mathbb{P}^1,\mathbb{P}^\textnormal{3};d^{\textnormal{1-3}} \big) +W_1\big(\mathbb{P}^\textnormal{3},\mathbb{P}^2;d^{\textnormal{3-2}} \big).
    \end{align}
    Here, step~$(a)$ stems from the assumption on $d$. We have now established the transitivity of the inter-MDP triangle inequality on the Wasserstein distance.

    Armed with (\ref{preservation}), we are ready to prove the inter-MDP triangle inequality of the GBSM through induction. For the base case,
    \begin{align}
         d_0^{\textnormal{1-2}}(s,s')\leq \ & d_0^{\textnormal{1-3}}(s,s'')+d_0^{\textnormal{3-2}}(s'',s')=0,\notag\\
        &  \forall\ (s,s',s'')\in \mathcal{S}_1\times\mathcal{S}_2\times\mathcal{S}_3.
    \end{align}
    By the induction hypothesis, we assume that for an arbitrary $n\in\mathbb{N}$ and for any $(s,s',s'')\in \mathcal{S}_1\times\mathcal{S}_2\times\mathcal{S}_3$,
    \begin{align}
             &d_n^{\textnormal{1-2}}(s,s')\leq d_n^{\textnormal{1-3}}(s,s'')+d_n^{\textnormal{3-2}}(s'',s'),\
    \end{align}
    The induction follows
    \begin{align}
          d_{n+1}^{\textnormal{1-2}}(s,s')  {=}&\ H(X_{s},X_{s'};\delta (d_{n}^{\textnormal{1-2}}))\notag\\
{=}&\ \max\Big\{\max_{a\in\mathcal{A}_1} \min_{a'\in\mathcal{A}_2}\delta(d_{n}^{\textnormal{1-2}})((s,a),(s',a')),\notag\\
&\qquad\quad \min_{a\in\mathcal{A}_1}\max_{a'\in\mathcal{A}_2}\delta(d_{n}^{\textnormal{1-2}})((s,a),(s',a'))\Big\}
    \end{align}
    Without loss of generality, we assume that $\max_{a\in\mathcal{A}_1} \allowbreak \min_{a'\in\mathcal{A}_2}\delta(d_{n}^{\textnormal{1-2}})((s,a),(s',a'))$ is greater, then we have
    \begin{align}
            &\ d_{n+1}^{\textnormal{1-2}}(s,s') \notag\\
            {=} &\ \max_{a\in\mathcal{A}_1} \min_{a'\in\mathcal{A}_2}\delta(d_{n}^{\textnormal{1-2}})((s,a),(s',a'))\notag\\
            \overset{(a)}{\leq} &\ \max_{a\in\mathcal{A}_1}\min_{a'\in\mathcal{A}_2} \Big\{\big|R_1(s,a)-R_3(s'',a_{\min}'')\big|\notag\\
            &\qquad\qquad\quad\ + \big|R_3(s'',a''_{\min})-R_2(s',a')\big|\notag\\
            &\qquad\qquad\quad\  + \gamma W_1\big(\mathbb{P}_1(\cdot|s,a), \mathbb{P}_3(\cdot|s'',a''_{\min})  ;d_n^{\textnormal{1-3}}\big) \notag\\
            &\qquad\qquad\quad\ + \gamma W_1\big( \mathbb{P}_3(\cdot|s'',a''_{\min}), \mathbb{P}_2(\cdot|s',a') ;d_n^{\textnormal{3-2}}\big)\Big\}\notag\\
            {\leq} &\ \max_{a\in\mathcal{A}_1}\min_{a''\in\mathcal{A}_3} \Big\{\big|R_1(s,a)-R_3(s'',a'')\big|\notag\\
            &\qquad\qquad\quad\  + \gamma W_1\big(\mathbb{P}_1(\cdot|s,a), \mathbb{P}_3(\cdot|s'',a'')  ;d_n^{\textnormal{1-3}}\big)\Big\} \notag\\
            &\ + \min_{a'\in\mathcal{A}_2} \Big\{\big|R_3(s'',a''_{\min})-R_2(s',a')\big|\notag\\
            &\qquad\qquad\quad\ + \gamma W_1\big( \mathbb{P}_3(\cdot|s'',a''_{\min}), \mathbb{P}_2(\cdot|s',a') ;d_n^{\textnormal{3-2}}\big)\Big\}\notag\\
            \leq &\ \max_{a\in\mathcal{A}_1}\min_{a''\in\mathcal{A}_3} \delta(d_{n}^{\textnormal{1-3}})((s,a),(s'',a''))\notag\\
            &\ +\max_{a''\in\mathcal{A}_3}\min_{a'\in\mathcal{A}_2}\delta(d_{n}^{\textnormal{3-2}})((s'',a''),(s',a')) \notag\\  
            \leq &\ H(X_s,X_{s''};\delta (d_{n}^{\textnormal{1-3}})+H(X_{s''},X_{s'};\delta (d_{n}^{\textnormal{3-2}}) \notag\\  
            = &\ d_{n+1}^{\textnormal{1-3}}(s,s'') + d_{n+1}^{\textnormal{3-2}}(s'',s'),\  \forall\ (s,s',s'')\in \mathcal{S}_1\times\mathcal{S}_2\times\mathcal{S}_3. 
    \end{align}
    Here, step~$(a)$ follows from (\ref{preservation}), the transitivity of the inequality, and $a_{\min}'' = \arg\min_{a''\in \mathcal{A}_3}  \{\max_{a\in\mathcal{A}_1} \delta(d_n^{\textnormal{1-3}})((s,a),\allowbreak(s'',a''))\}$. 
    Now we have $d_n^{\textnormal{1-2}}(s,s')\leq d_n^{\textnormal{1-3}}(s,s'')+d_n^{\textnormal{3-2}}(s'',s')$ for all $n\in\mathbb{N}$. Taking $n\rightarrow \infty$, we establish the inter-MDP triangle inequality of GBSM.

\section{Proof of Theorem \ref{Theorem:5}}\label{Appendix:Theorem:5}
Consider a special transportation plan between distributions $P$ and $Q$ over the same state space. This plan preserves all mass shared between $P$ and $Q$, defined as $\min\{P(s),Q(s)\}$ for all $s$. The remaining mass, where $P(s)>Q(s)$, is distributed to states where $P(s)<Q(s)$. Then the total mass to be transported is quantified by the total variation distance, where the transportation cost with the cost function $d$ is bounded by $\max_{s,s'}{d}(s,s')$. The shared mass is given by $1-\textnormal{TV}(P,Q)$, with the cost bounded by $\max_{s}{d}(s,s)$. While this plan adheres to the definition of Wasserstein distance, it is unlikely to be the optimal plan that yields the minimum transportation cost. Then we have
\begin{align}\label{tvandW}
 W_1(P,Q;d)\leq &\ \textnormal{TV}(P,Q) \max_{s,s'}{d}(s,s')  \notag\\
&\ + \big(1-\textnormal{TV}(P,Q)\big)\max_{s}{d}(s,s).
\end{align}
Also, for any state $s\in\mathcal{S}_1$ and $s'\in\mathcal{S}_2$, we have
\begin{align}
    d^{1\text{-}2}(s, s')\leq&\  \max_{a,a'} \Big\{\big| R_1(s,a)-R_2(s',a') \big|  \notag\\
    &\ +\gamma W_1\big(\mathbb{P}_1(\cdot|s,a), \mathbb{P}_2(\cdot|s',a') ;d^{1\text{-}2}\big)\Big\} \notag\\
    \leq &\  \bar{R}+ \gamma\max_{s,s'}d^{1\text{-}2}(s,s').
\end{align}
Rearranging the inequality yields
\begin{equation}\label{GBSMmaximum}
    \max_{s,s'}d^{1\text{-}2}(s,s')\leq \bar{R}/(1-\gamma),\  \forall\ (s_i,s_j)\in \mathcal{S}_1\times\mathcal{S}_2.
\end{equation}
Without loss of generality, we assume that $\max_{s}d^{\textnormal{1-2}}(s,s)=\max_{s,a}\allowbreak \min_{a'}\delta(d^{\textnormal{1-2}})((s,a),(s,a'))$, then
\begin{align}
    &\max_{s}d^{\textnormal{1-2}}(s,s)\notag\\
{=}&\max_{s,a}\min_{a'}  \Big\{ |R_1(s,a) - R_2(s,a')|  \notag\\
& \ \qquad\qquad+ \gamma W_1\big(\mathbb{P}_1(\cdot|s,a), \mathbb{P}_2(\cdot|s,a')  ;d^{\textnormal{1-2}}\big)\big\}  \notag\\
\overset{(a)}{\leq}&\max_{s,a}\min_{a'}  \Big\{ |R_1(s,a) - R_2(s,a')|  \notag\\
& \ +  \gamma\textnormal{TV}\big(\mathbb{P}_1(\cdot|s,a), \mathbb{P}_2(\cdot|s,a')\big)\max_{s,s'}d^{\textnormal{1-2}}(s,s') \notag\\
&\ +\gamma \max_{s,a}  \big(1-\textnormal{TV}\big(\mathbb{P}_1(\cdot|s,a), \mathbb{P}_2(\cdot|s,a) \big)\big)\max_{s}d^{\textnormal{1-2}}(s,s)  \Big\} \notag\\
\overset{(b)}{\leq}&\max_{s,a}\min_{a'}  \Big\{ |R_1(s,a) - R_2(s,a')|  \notag\\
& \ +  \frac{\gamma \bar{R}}{1\!-\!\gamma}\textnormal{TV}\big(\mathbb{P}_1(\cdot|s,a), \mathbb{P}_2(\cdot|s,a')\big) \Big\} +\gamma \max_{s}d^{\textnormal{1-2}}(s,s)  \notag\\
{\leq}& \max_{s} H(X_s,X_s;\delta_{\textnormal{TV}})  +  \!\gamma\max_{s}d^{\textnormal{1-2}}(s,s).
\end{align}
where step~$(a)$ uses (\ref{tvandW}) and step~$(b)$ stems from (\ref{GBSMmaximum}). Rearranging the inequality yields the desired result.

\section{Proof of Theorem \ref{Theorem:6}}\label{Appendix3}
This section provides a detailed proof of the regret bound for policy $\pi$ transferred from $\mathcal{M}_1$ to $\mathcal{M}_2$.
By the triangle inequality, for any state $s_j\in\mathcal{S}_2$ and $s_i=f(s_j)\in\mathcal{S}_1$, we have
\begin{align}\label{eq33}
    |V_2^*(s_j)-V_2^{\pi}(s_j)|\leq & \  |V_2^*(s_j)-V_1^*(s_i)| +|V_1^*(s_i)-V_1^{\pi}(s_i)|\notag\\
    &\  +|V_1^{\pi}(s_i)-V_2^{\pi}(s_j)|.
\end{align}
Within the right-hand side of this inequality, the first summation term $|V_2^*(s_j)-V_1^*(s_i)|$ is upper bounded by ${d}^{1\text{-}2}(s_i,s_j)$ according to Theorem \ref{Theorem:2}, and $|V_1^*(s_i)-V_1^{\pi}(s_i)|$ is upper bounded by $\max_{s\in\mathcal{S}_1}|V_1^*(s)-V_1^{\pi}(s)|$. For the last term, we have
\begin{align}
 &\big|V_1^\pi(s_i)-V_2^{\pi}(s_j)\big|\notag\\
{=}&\Bigg| \sum_{a=1}^{|\mathcal{A}_1|} \pi(a|s_i)\bigg(R_1(s_i,a)+\gamma\sum_{k=1}^{|\mathcal{S}_1|} \mathbb{P}_1(s_k|s_i,a) V_1^\pi(s_k)\bigg) \notag\\
 &-\!\sum_{a=1}^{|\mathcal{A}_1|}\! \pi(a|f(s_j))\! \bigg(\!R_2(s_j,g(a)\!)\!+\!\gamma\!\sum_{k=1}^{|\mathcal{S}_2|}\!\! \mathbb{P}_2(s_k|s_j,g(a)) V_2^\pi(s_k)\!\bigg) \!\Bigg|\notag\\
 {\leq} &\sum_{a=1}^{|\mathcal{A}_1|} \Bigg(\pi(a|s_i) \Bigg(\Big| R_1 (s_i,a)-R_2(s_j,g(a))\Big| \notag\\
 &+\gamma\bigg|\sum_{k=1}^{|\mathcal{S}_1|}\mathbb{P}_1(s_k|s_i,a) V_1^\pi(s_k) \!-\! \sum_{k=1}^{|\mathcal{S}_2|} \mathbb{P}_2(s_k|s_j,g(a)) V_2^\pi(s_k) \bigg|\Bigg)\!\Bigg)\notag\\
 {\leq} &\sum_{a=1}^{|\mathcal{A}_1|} \Bigg(\pi(a|s_i) \Bigg(\Big| R_1 (s_i,a)-R_2(s_j,g(a))\Big|\notag\\
 &+\gamma\bigg|\sum_{k=1}^{|\mathcal{S}_1|}\mathbb{P}_1(s_k|s_i,a) V_1^*(s_k)\! -\! \sum_{k=1}^{|\mathcal{S}_2|} \mathbb{P}_2(s_k|s_j,g(a)) V_2^*(s_k) \bigg|\Bigg)\!\Bigg) \notag\\
 &+\gamma \sum_{a=1}^{|\mathcal{A}_1|} \pi(a|s_i) \bigg|\sum_{k=1}^{|\mathcal{S}_1|} \mathbb{P}_1(s_k|s_i,a) (V_1^*(s_k)- V_1^\pi(s_k))\bigg|\notag\\
 &+\gamma \sum_{a=1}^{|\mathcal{A}_1|} \pi(a|s_i) \bigg|\sum_{k=1}^{|\mathcal{S}_2|} \mathbb{P}_2(s_k|s_j,g(a)) (V_2^*(s_k)-V_2^\pi(s_k) )\bigg| \notag\\
 \leq &\max_{a\in\mathcal{A}_1} \Bigg\{\Big| R_1 (s_i,a) -R_2(s_j,g(a))\Big|\notag\\
 &+\gamma\bigg|\sum_{k=1}^{|\mathcal{S}_1|} \mathbb{P}_1(s_k|s_i,a) V_1^*(s_k)\!-\!\sum_{k=1}^{|\mathcal{S}_2|} \mathbb{P} (s_k|s_j,g(a)) V_2^*(s_k)\bigg|\Bigg\} \notag\\
 & + \gamma \max_{s\in\mathcal{S}_1} |V_1^*(s) - V_1^\pi(s)|+ \gamma \max_{s\in\mathcal{S}_2} |V_2^*(s) - V_2^\pi(s)| \notag\\
 \overset{(a)}{\leq} &\max_{a\in\mathcal{A}_1} \bigg\{\Big| R_1 (s_i,a) -R_2(s_j,g(a))\Big|\notag\\
 &\ \qquad+\gamma W_1\big(\mathbb{P}_1(\cdot|s_i,a), \mathbb{P}_2(\cdot|s_j,a');d^{1\text{-}2}\big)\bigg\} \notag\\
 & + \gamma \max_{s\in\mathcal{S}_1} |V_1^*(s) - V_1^\pi(s)|+ \gamma \max_{s\in\mathcal{S}_2} |V_2^*(s) - V_2^\pi(s)| \notag\\
 = &\max_{a\in\mathcal{A}_1} \delta(d^{\textnormal{1-2}})((s_i,a),(s_j,g(a)))  \notag\\
 & + \gamma \max_{s\in\mathcal{S}_1} |V_1^*(s) - V_1^\pi(s)|+ \gamma \max_{s\in\mathcal{S}_2} |V_2^*(s) - V_2^\pi(s)| 
\end{align}
Here, step~$(a)$ stems from the fact that, according to Theorem \ref{Theorem:2}, $\big( V_1^*(s_k)\big)_{k=1}^{|\mathcal{S}_1|}$ and $\big( V_2^*(s_k) \big)_{k=1}^{|\mathcal{S}_2|}$ form a feasible, but not necessarily the optimal, solution to the dual LP in (\ref{LP2}) for $W_1\big(\mathbb{P}_1(\cdot|s_i,a), \mathbb{P}_2(\cdot|s_j,a);d^{1\text{-}2}\big)$. Combining the above inequalities on all three summation terms in (\ref{eq33}) and taking the maximum of both sides, we have 
\begin{align}
& \max_{s'\in\mathcal{S}_2}|V_2^*(s')-V_2^{\pi}(s')| \notag\\
     \leq & \underbrace{\max_{s'\in\mathcal{S}_2}d^{1\text{-}2}(f(s'),s')}_{\text{1st term}}+ \underbrace{\max_{s\in\mathcal{S}_1}|V_1^*(s)-V_1^{\pi}(s)|}_{\text{2nd term}} \notag\\
   &+  \underbrace{\max_{s'\in\mathcal{S}_2,a\in\mathcal{A}_1} \delta(d^{\textnormal{1-2}})((s_i,a),(s_j,g(a))) }_{\text{3rd term}} \notag\\
    &\underbrace{+ \gamma \max_{s\in\mathcal{S}_1} \Big|V_1^*(s) - V_1^\pi(s)  \Big|+ \gamma \max_{s'\in\mathcal{S}_2} \Big|V_2^*(s') - V_2^\pi(s')\Big|}_{\text{3rd term}}\notag\\
    \leq&\ \max_{s\in\mathcal{S}_2}d^{1\text{-}2}(f(s),s) + \max_{s'\in\mathcal{S}_2,a\in\mathcal{A}_1} \delta(d^{\textnormal{1-2}})((s_i,a),(s_j,g(a)))\notag\\
    &+\!(1\!+\!\gamma)\max_{s\in\mathcal{S}_1}|V_1^*(s)\!-\!V_1^{\pi}(s)| \!+\!\gamma \max_{s'\in\mathcal{S}_2} \Big|V_2^*(s') \!-\! V_2^\pi(s')\Big|.
\end{align}
Rearranging the inequality yields the desired result.

\section{Proof of Theorem \ref{Theorem:vfa}}\label{Appendix:Theorem:vfa}
The first inequality in (\ref{vfa}) is a direct consequence of Theorem~\ref{Theorem:2}. For the second one, we construct an intermediate MDP defined by $\mathcal{M}_{1_{[\mathcal{S}]}}= \langle [\mathcal{S}_1], \mathcal{A}_1, \mathbb{P}_1, R_1, \gamma\rangle$, where the reward function and transition probability of each state are replaced by the counterparts of its representative state, and construct a similar $\mathcal{M}_{2_{[\mathcal{S}]}}$. Without loss of generality, taking $\mathcal{M}_{1}$ as an example, we now prove $d^{1\textnormal{-}1_{[\mathcal{S}]}}(s,s)=d^{1\textnormal{-}[1]}(s,[s])$ for all $s\in\mathcal{S}_1$ through induction. For the base case, $d^{1_{[\mathcal{S}]}\textnormal{-}[1]}_1(s,[s])=\max_a|R_1([s],a)-R_1([s],a)|=0$. By the induction hypothesis, we assume that $d^{1_{[\mathcal{S}]}\textnormal{-}[1]}_n(s,[s])=0$ for any $n$, then
\begin{align}
    &\ d^{1_{[\mathcal{S}]}\textnormal{-}[1]}_{n+1}(s,[s])\notag\\
    =&\  H(X_s,X_{[s]};\delta (d^{1_{[\mathcal{S}]}\textnormal{-}[1]}_{n}))  \notag\\
    \leq&\  \max_a \Big\{ |R_1([s],a)-R_1([s],a)|  \notag\\
&\qquad\ \ \ + \gamma W_1(\mathbb{P}(\cdot|[s],a),[\mathbb{P}](\cdot|[s],a);d^{1_{[\mathcal{S}]}\textnormal{-}[1]}_n) \Big\} \notag\\
    \leq&\  \gamma \max_a \sum_{\tilde{s}\in\mathcal{S}_1} \mathbb{P}(\tilde{s}|[\tilde{s}],a) d^{1_{[\mathcal{S}]}\textnormal{-}[1]}_n(\tilde{s},[\tilde{s}])=0.
\end{align}
The second inequality follows from a transportation plan that moves the mass from each $\tilde{s}$ to its representative state $[\tilde{s}]$. Now we have established $d^{1_{[\mathcal{S}]}\textnormal{-}[1]}_{n}(s,[s])=0$ for all $n\in\mathbb{N}$ and $s\in\mathcal{S}_1$. Taking $n\rightarrow \infty$, we have $d^{1_{[\mathcal{S}]}\textnormal{-}[1]}(s,[s])=0,\ \forall s\in\mathcal{S}_1$. Using the inter-MDP triangle inequality in Theorem \ref{Theorem:4}, we derive $ d^{1\textnormal{-}1_{[\mathcal{S}]}}(s,s)=d^{1\textnormal{-}[1]}(s,[s])$ for all $s\in\mathcal{S}_1$. Meanwhile, by observing the special case $d^{1_{[\mathcal{S}]}\textnormal{-}[1]}([s],[s])=0$, we also have $ d^{1\textnormal{-}1_{[\mathcal{S}]}}(s,[s])=d^{1\textnormal{-}[1]}(s,[s])$ through the inter-MDP triangle inequality. Then, we have
\begin{align}
 d^{1\textnormal{-}1_{[\mathcal{S}]}}(s,s)= d^{1\textnormal{-}1_{[\mathcal{S}]}}(s,[s])=d^{1\textnormal{-}[1]}(s,[s]),\ \forall s\in\mathcal{S}_1.   
\end{align}

Next, we prove the inequality between $\sigma_1$ and $\tilde{\sigma}_1$. First, we have
\begin{align}
    &\ d^{1\textnormal{-}1_{[\mathcal{S}]}}([s],[s]) \notag\\
=&\  H(X_{[s]},X_{[s]};\delta (d^{1\textnormal{-}1_{[\mathcal{S}]}}))  \notag\\
\leq &\ \max_a \Big\{ |R_1([s],a)-R_1([s],a)| \notag\\
&\qquad\ \ + \gamma W_1(\mathbb{P}(\cdot|[s],a),\mathbb{P}(\cdot|[s],a);d^{1\textnormal{-}1_{[\mathcal{S}]}}) \Big\}\notag\\
     = &\ \gamma\max_a \Big\{W_1(\mathbb{P}(\cdot|[s],a),\mathbb{P}(\cdot|[s],a);d^{1\textnormal{-}1_{[\mathcal{S}]}}) \Big\}\notag\\
     {\leq} &\ \gamma \max_a \Big\{\sum_{\tilde{s}\in\mathcal{S}} \mathbb{P}(\tilde{s}|[s],a) d^{1\textnormal{-}1_{[\mathcal{S}]}}(\tilde{s},\tilde{s}) \Big\}\notag\\
     \leq &\ \gamma \max_s d^{1\textnormal{-}1_{[\mathcal{S}]}}(s,s) = \gamma \max_s d^{1\textnormal{-}1_{[\mathcal{S}]}}(s,[s]).
\end{align}
Here, the first inequality follows from a straightforward transportation plan that keeps all the mass at its position. Then, according to the inter-MDP triangle inequality in Theorem \ref{Theorem:4}, we have 
\begin{align}
    d^{1\textnormal{-}1_{[\mathcal{S}]}}(s,[s]) &\leq d^{1\textnormal{-}1}(s,[s]) + d^{1\textnormal{-}1_{[\mathcal{S}]}}([s],[s])\notag\\
&\leq d^{1\textnormal{-}1}(s,[s]) + \gamma \max_s d^{1\textnormal{-}1_{[\mathcal{S}]}}(s,[s]).
\end{align}
Taking the maximum of both sides, rearranging the inequality, and combining the established $d^{1\textnormal{-}1_{[\mathcal{S}]}}(s,s)=d^{1\textnormal{-}[1]}(s,s)$, we finally have
\begin{align}
    \sigma_1&\ =\max_{s}d^{1\textnormal{-}[1]}(s,[s]) \notag\\
    &\ = \max_{s}d^{1\textnormal{-}1_{[\mathcal{S}]}}(s,[s])\notag\\
    &\ \leq \max_{s}d^{1\textnormal{-}1}(s,[s]) /(1\!-\!\gamma) \notag\\
&\ \leq \max_{s}d^{\sim}(s,[s]) /(1-\gamma) = \tilde{\sigma}_1/(1-\gamma).
\end{align}

\section{Proof for Tightness of GBSM}\label{Appendix1.5}
First, we recall the original definition of GBSM in the earlier conference version \cite{GBSM}.
\begin{defi}[\textbf{Original GBSM definition}]
Given two MDPs $\mathcal{M}_1=\langle \mathcal{S}_1, \mathcal{A}, \mathbb{P}_1, R_1, \gamma\rangle$ and $\mathcal{M}_2=\langle \mathcal{S}_2, \mathcal{A}, \mathbb{P}_2, R_2, \gamma\rangle$, the GBSM between any states $s\in\mathcal{S}_1$ and $s'\in\mathcal{S}_2$ is defined as:
\begin{align}
    d_\textnormal{conf}^{1\textnormal{-}2}(s,s')=&\ \max_{a} \Big\{\big|R_1(s,a)-R_2(s',a)\big| \notag\\
    &\ \qquad\ \ +\gamma W_1(\mathbb{P}_1(\cdot|s,a), \mathbb{P}_2(\cdot|s',a) ;d_\textnormal{conf}^{1\text{-}2})\Big\} \notag\\
    =&\ \max_a \delta(d_{\textnormal{conf}}^{1\text{-}2})((s,a), (s',a)).
\end{align}
\end{defi}
Similarly, when writen iteratively as $d_{\textnormal{conf},n}^{1\textnormal{-}2}(s,s')= \max_a \delta(d_{\textnormal{conf},n-1}^{1\text{-}2})((s,a), (s',a))$, $d_{\textnormal{conf},n}$ converges to the fixed point $d_\textnormal{conf}^{1\textnormal{-}2}$ with $n\rightarrow\infty$. Then we are ready to prove the tightness of the GBSM proposed in this paper and its associated bounds compared with those of $d_\textnormal{conf}^{1\textnormal{-}2}$.
\begin{lemma}[\textbf{Tightness of $d^{1\textnormal{-}2}$}]\label{tightgbsm}
When $\mathcal{M}_1$ and $\mathcal{M}_2$ share the same $\mathcal{A}$, for any $s\in \mathcal{S}_1$ and $s'\in \mathcal{S}_2$,
    \begin{align}
        d^{1\textnormal{-}2}(s,s') \leq {d}_\textnormal{conf}^{1\text{-}2}(s,s').
    \end{align}
\end{lemma}
\begin{proof}
For the base case, we have
        \begin{align}
            d_0^{1\textnormal{-}2}(s,s') = d_\textnormal{conf,0}^{1\textnormal{-}2}(s,s') = 0,\  \forall\ (s,s')\in \mathcal{S}_1\times\mathcal{S}_2.
        \end{align}
    By the induction hypothesis, we assume that for an arbitrary $n$,
    \begin{align}
        d^{1\text{-}2}_n(s,s') \leq d^{1\text{-}2}_{\textnormal{conf},n}(s,s'),\ \forall\ (s,s')\in \mathcal{S}_1\times\mathcal{S}_2.
    \end{align}
    By the continuity of $F$, we have that $F(d_\textnormal{n}^{1\text{-}2}|s,s') \leq F(d_{\textnormal{conf},n}^{1\text{-}2}|s,s')$, which implies
    \begin{align}
         d_\textnormal{n+1}^{1\text{-}2}(s,s')\leq&\ F(d_{\textnormal{conf},n}^{1\text{-}2}|s,s') \notag\\
        =&\max \Big\{ \max_{a\in\mathcal{A}} \min_{a'\in\mathcal{A}} \delta(d_{\textnormal{conf},n}^{1\text{-}2})((s,a), (s',a')) ,\notag\\
        &\qquad\quad \min_{a\in\mathcal{A}} \max_{a'\in\mathcal{A}} \delta(d_{\textnormal{conf},n}^{1\text{-}2})((s,a), (s',a')) \Big\}\notag\\
        \leq&\ \max_a \delta(d_{\textnormal{conf},n}^{1\text{-}2})((s,a), (s',a)) \notag\\
        =&\ d_{\textnormal{conf},n+1}^{1\text{-}2}(s,s'),\ \forall\ (s,s')\in \mathcal{S}_1\times\mathcal{S}_2.
    \end{align}
    Taking $n\rightarrow \infty$ yields the desired result.
\end{proof}

\begin{lemma}[\textbf{Tightness of $d^{1\textnormal{-}2}$-based policy transfer regret bound}]
When $\mathcal{M}_1$ and $\mathcal{M}_2$ share the same $\mathcal{A}$, and the action mapping is the identity mapping $g(a)=a$, the regret of the transferred policy in $\mathcal{M}_2$ is bounded by:
 \begin{align}
    &\ \ \max_{s'\in\mathcal{S}_2}|V_2^*(s')-V_2^{\pi}(s')|\notag\\
    &\underbrace{\begin{aligned}
        \leq & \big(\max_{s'\in\mathcal{S}_2}d^{\textnormal{1-2}}(f(s'),s')+\max_{\substack{s' \in \mathcal{S}_2 \\ a \in \mathcal{A}}}\delta(d^{\textnormal{1-2}})((f(s'),a),(s',g(a)))\notag\\
        &+(1+\gamma)\max_{s\in\mathcal{S}_1}|V_1^*(s)-V_1^{\pi}(s)|\big)\big/\big(1-\gamma\big)
    \end{aligned}}_{d^{1\textnormal{-}2}\text{-based bound}}
    \notag\\
    &\underbrace{\begin{aligned}
         \leq  \frac2{1-\gamma}\max_{s'\in\mathcal{S}_2}d_{\textnormal{conf}}(f(s'),s')+\frac{1+\gamma}{1-\gamma}\max_{s\in\mathcal{S}_1}|V_1^*(s)-V_1^{\pi}(s)|.
    \end{aligned}}_{d_{\textnormal{conf}}^{1\textnormal{-}2}\text{-based bound}}
\end{align}
\end{lemma}
\begin{proof}
We first derive
\begin{equation}\label{D.1.1}
    \max_{s'\in\mathcal{S}_2}d^{\textnormal{1-2}}(f(s'),s') \leq \max_{s'\in\mathcal{S}_2}d_{\textnormal{conf}}(f(s'),s')
\end{equation}
from Lemma~\ref{tightgbsm}, and substitute $g(a)=a$ into the second term of $d^{1\textnormal{-}2}$-based bound as follows.
\begin{align}\label{D.1.2}
    &\max_{s',a}\delta(d^{\textnormal{1-2}})((f(s'),a),(s',g(a)))\notag\\
   = & \max_{s',a}\big\{|R_1((f(s'),a)-R_2(s',a)| \notag\\
    &\qquad\  +\gamma W_1(\mathbb{P}_1(\cdot|(f(s'),a), \mathbb{P}_2(\cdot|s',a) ;d^{\textnormal{1-2}})\big\},\notag\\
   \leq & \max_{s',a}\big\{|R_1((f(s'),a)-R_2(s',a)| \notag\\
    &\qquad\ +\gamma W_1(\mathbb{P}_1(\cdot|(f(s'),a), \mathbb{P}_2(\cdot|s',a) ;d_{\textnormal{conf}}^{\textnormal{1-2}})\big\}\notag\\
    = &\max_{s'\in\mathcal{S}_2}d_{\textnormal{conf}^{\textnormal{1-2}}}(f(s'),s').
\end{align}
Combining (\ref{D.1.1}) and (\ref{D.1.2}) yields the desired result.
\end{proof}

\begin{lemma}[\textbf{Tightness of $d^{1\textnormal{-}2}$-based VFA bound}]
Given an MDP $\mathcal{M}_1$ and its aggregated counterpart $\mathcal{M}_{[1]}$, the VFA bound is given by
\begin{equation}
\max_{s\in\mathcal{S}_1}|V^*_1(s)-V^*_{[1]}(s)|\leq \sigma_1 \leq \sigma_{1,\textnormal{conf}},
\end{equation}
where $\sigma_1 = \max\limits_{s\in\mathcal{S}_1}d^{1\textnormal{-}[1]}(s,[s])$ and $\sigma_{1,\textnormal{conf}} = \max\limits_{s\in\mathcal{S}_1}d_\textnormal{conf}^{1\textnormal{-}[1]}(s,[s])$.
\end{lemma}
\begin{proof}
This inequality follows directly from Theorem~\ref{Theorem:vfa} and Lemma~\ref{tightgbsm}.
\end{proof}

\section{Proofs of on-policy GBSM Properties}\label{Appendix6}
\begin{theorem}[\textbf{On-policy GBSM optimal value difference bound}]\label{the:on-policyVB}
Let $V_1^\pi$ and $V_2^\pi$ denote the value functions with policy $\pi$ in $\mathcal{M}_1$ and $\mathcal{M}_2$, respectively. Then on-policy GBSM provides an upper bound for the difference between the value functions for any state pair $(s,s')\in \mathcal{S}_1\times\mathcal{S}_2$:
    \begin{align}
        |V^\pi_1(s)-V^\pi_2(s')|\leq d_\pi^{1\textnormal{-}2}(s,s').
    \end{align}
\end{theorem}
\begin{proof}
For the base case, we have
        \begin{align}\label{D1}
            &\ |V_1^{\pi,0} (s_i)-V_2^{\pi,0} (s_j)|  \notag\\
           {=} &\  d_{\pi,0}^{1\textnormal{-}2}(s_i,s_j)=0,\ \forall\ (s_i,s_j)\in \mathcal{S}_1\times\mathcal{S}_2.
        \end{align}
    Then, we assume that for an arbitrary $n$,
    \begin{align}\label{D2}
        |V_1^{\pi,n} (s_i)-V_2^{\pi,n} (s_j)| \leq&\ d^{1\text{-}2}_{\pi,n}(s_i,s_j), \forall (s_i,s_j)\in \mathcal{S}_1\!\times\!\mathcal{S}_2.
    \end{align}
    The induction follows
\begin{align}
           &\ \big|V_1^{\pi,n+1} (s_i)-V_2^{\pi,n+1} (s_j)\big| \notag\\
           {=} &\ \bigg|\Big(R_1^\pi(s_i)+\gamma\sum_{k=1}^{|\mathcal{S}_1|}\mathbb{P}_1^\pi(s_k|s_i)V_1^{\pi,n}(s_k)\Big) \notag\\
           &\ -\Big(R^\pi_2(s_j)+\gamma\sum_{k=1}^{|\mathcal{S}_2|}\mathbb{P}^\pi_2(s_k|s_j)V_2^{\pi,n}(s_k)\Big)\bigg|\notag\\
           {\leq} &\  \bigg|R_1^\pi(s_i) - R_2^\pi (s_j)\bigg|  +\gamma\bigg|\sum\limits_{k=1}^{|\mathcal{S}_1|} \mathbb{P}^\pi_1 (s_k|s_i) V_2^{\pi,n}(s_k)\notag\\
           &\ - \sum\limits_{k=1}^{|\mathcal{S}_2|} \mathbb{P}^\pi_2(s_k|s_j) V_2^{\pi,n}(s_k)\bigg|\notag\\
           {\leq}&\  \Big|R_1^\pi(s_i)- R_2^\pi (s_j)\Big| +\gamma W_1\Big(\mathbb{P}^\pi_1(\cdot|s_i),\mathbb{P}^\pi_2(\cdot|s_j);d^{1\text{-}2}_n\Big)\notag\\
           {=}&\ d^{1\text{-}2}_{n+1}(s_i,s_j),\ \forall\ (s_i,s_j)\in \mathcal{S}_1\times\mathcal{S}_2.
        \end{align}

   Taking $n\rightarrow \infty$ yields the desired result.
\end{proof}

\begin{theorem}[\textbf{On-policy GBSM distance bound on identical spaces}]
When $\mathcal{M}_1$ and $\mathcal{M}_2$ share the same $\mathcal{S}$,
\begin{align}
   \max_{s}d^{\textnormal{1-2}}_\pi(s,s) \leq &\frac{1}{1-\gamma}\max_{s,a} \Big\{\big|R_1^\pi(s)-R_2^\pi(s)\big|\notag\\ 
   &\ \ \ \ +\frac{\gamma \bar{R}}{1-\gamma} \textnormal{TV}\left(\mathbb{P}^\pi_1(\cdot|s), \mathbb{P}^\pi_2(\cdot|s) \right)\Big\},
\end{align}
where TV represents the total variation distance defined by $\textnormal{TV}(P,Q)= \frac{1}{2} \sum_{s\in\mathcal{S}} \big|P(s)-Q(s)\big|$.
\end{theorem}
\begin{proof}
Using inequality (\ref{tvandW}) in Theorem~\ref{Theorem:5}, we have
\begin{align}
    &\ \max_{s}d_\pi^{\textnormal{1-2}}(s,s)\notag\\
{\leq}&\  |R_1^\pi(s) - R_2^\pi(s)|   +  \gamma\textnormal{TV}\big(\mathbb{P}^\pi_1(\cdot|s), \mathbb{P}^\pi_2(\cdot|s)\big)\max_{s,s'}d_\pi^{\textnormal{1-2}}(s,s') \notag\\
&\ +\gamma \big(1-\textnormal{TV}\big(\mathbb{P}^\pi_1(\cdot|s), \mathbb{P}^\pi_2(\cdot|s) \big)\big)\max_{s}d_\pi^{\textnormal{1-2}}(s,s)\notag\\
{\leq}& \ |R_1^\pi(s) - R_2^\pi(s)|   +  \frac{\gamma \bar{R}}{1-\gamma}\textnormal{TV}\big(\mathbb{P}^\pi_1(\cdot|s), \mathbb{P}^\pi_2(\cdot|s )\big)  \notag\\
&\ + \gamma\max_{s}d_\pi^{\textnormal{1-2}}(s,s).\notag
\end{align}
Rearranging the inequality yields the desired result.
\end{proof}

\begin{theorem}[\textbf{VFA error bound with non-optimal policy}]
    \begin{align}
        \max_s |V^\pi_1(s)-V^\pi_{[1]}(s)|\leq&\ \max_s d_\pi^{1\textnormal{-}[1]}(s,s)\notag\\
        \leq&\ \max_s \tilde{d}_\pi(s,[s])/(1-\gamma)
    \end{align}
\end{theorem}
\begin{proof}
The first inequality follows directly from Theorem~\ref{the:on-policyVB}, while the second is established using a derivation analogous to the proof of Theorem~\ref{Theorem:vfa}.
\end{proof}

\end{appendices}
\ifCLASSOPTIONcaptionsoff
  \newpage
\fi
\footnotesize
\bibliographystyle{IEEEtran}

\bibliography{IEEEabrv,IEEEexample}
\end{document}